\documentclass[10pt]{article} 
\usepackage[accepted]{tmlr}

\usepackage{amsmath,amsfonts,bm}

\def\Eqref#1{Equation~\ref{#1}}

\def\1{\bm{1}}

\DeclareMathAlphabet{\mathsfit}{\encodingdefault}{\sfdefault}{m}{sl}
\SetMathAlphabet{\mathsfit}{bold}{\encodingdefault}{\sfdefault}{bx}{n}

\DeclareMathOperator*{\argmax}{arg\,max}

\usepackage{hyperref}
\usepackage{url}
\usepackage{natbib}
\usepackage{booktabs}
\usepackage{amssymb}

\usepackage{algorithm}
\usepackage{graphicx}
\usepackage{booktabs} 

\usepackage{amsthm}
\usepackage[capitalize,noabbrev]{cleveref}

\usepackage{xcolor}
\usepackage{soul}

\theoremstyle{plain}
\newtheorem{theorem}{Theorem}

\newtheorem{lemma}{Lemma}
\newtheorem{corollary}{Corollary}
\theoremstyle{definition}
\newtheorem{definition}{Definition}

\theoremstyle{remark}
\newtheorem{remark}{Remark}

\newcommand{\shw}[1]{{\color{purple}#1}}

\newcommand{\blue}{\color{blue}}
\newcommand{\black}{\color{black}}

\def\discussionSCalgchallenges{0}

\usepackage{amsmath}
\usepackage{mathtools}
\usepackage{enumitem}
\usepackage{thm-restate}
\usepackage{mathtools}
\usepackage{tcolorbox}
\usepackage{xcolor}
\usepackage{graphicx} 
\usepackage{algorithmic}
\newcount\Comments  
\newcommand{\kibitz}[2]{\ifnum\Comments=1{\textcolor{#1}{{#2}}}\fi}

\definecolor{darkgreen}{rgb}{0,0.6,0}

	    \newcommand{\red}[1]{{\leavevmode\color{red}{#1}}}

\newcommand{\MINTSS}{\texttt{MINTSS}}
\renewcommand{\cite}[1]{\citep{#1}}

\title{A Resilience Framework for Bi-Criteria Combinatorial Optimization with Bandit Feedback}

\author{\name Vaneet Aggarwal\thanks{Authors are listed in alphabetical order.} \email vaneet@purdue.edu \\
\addr Purdue University
\ANDD
\name Shweta Jain \email shwetajain@iitrpr.ac.in \\
\addr IIT Ropar
\ANDD
\name Subham Pokhriyal \email subham.22csz0002@iitrpr.ac.in\\
\addr IIT Ropar
\ANDD
\name Christopher John Quinn \email cjquinn@iastate.edu\\
\addr Iowa State University
}

\def\month{05}  
\def\year{2026} 
\def\openreview{\url{https://openreview.net/forum?id=jcjxXUMyJ5}} 

\begin{document}

\maketitle

\begin{abstract}
We study bi-criteria combinatorial optimization under noisy function evaluations. While resilience and black-box offline-to-online reductions have been studied in single-objective settings, extending these ideas to bi-criteria problems introduces new challenges due to the coupled degradation of approximation guarantees for objectives and constraints. We introduce a notion of $(\alpha,\beta,\delta,\texttt{N})$-resilience for bi-criteria approximation algorithms, capturing how joint approximation guarantees degrade under bounded (possibly worst-case) oracle noise, and develop a general black-box framework that converts any resilient offline algorithm into an online algorithm for bi-criteria combinatorial multi-armed bandits with bandit feedback. The resulting online guarantees achieve sublinear regret and cumulative constraint violation of order $\tilde{O}(\delta^{2/3}\texttt{N}^{1/3}T^{2/3})$ without requiring structural assumptions such as linearity, submodularity, or semi-bandit feedback on the noisy functions. We demonstrate the applicability of the framework by establishing resilience for several classical greedy algorithms in submodular optimization.
\end{abstract}








\section{Introduction}\label{sec:intro}

Many real-world combinatorial optimization problems arising in multi-agent and learning systems require balancing competing objectives, such as minimizing cost while ensuring coverage, utility, or fairness constraints. Classical examples include sensor placement and experimental design \cite{krause2005near,krause2008near}, feature subset selection \cite{krause2005near}, and fair resource allocation \cite{ogryczak2010bicriteria}. These problems are naturally modeled as \emph{bi-criteria} combinatorial optimization tasks, where one objective must be optimized subject to constraints on another.

A common approach to handling multiple objectives is to reduce them to a single objective, for example by minimizing the difference of two submodular functions (DS optimization). However, DS optimization is known to be NP-hard and inapproximable in general \cite{trevisan2014inapproximability}. As a result, much of the literature formulates bi-criteria problems as constrained combinatorial optimization tasks, such as minimizing a cost function subject to utility constraints. Since these problems are typically NP-hard, prior work has focused on designing approximation algorithms that achieve \emph{bi-criteria guarantees}, relaxing both the objective and the constraint by controlled approximation factors \cite{iyer2012algorithms,crawford2019submodular,chen2024fair,chen2024bicriteria}. Existing analyses, however, almost universally assume access to \emph{exact} value oracles for the objective and constraint functions \cite{wolsey1982analysis,wan2010greedy,soma2015generalization}.

In practice, exact oracle access is rarely available. Objective and constraint evaluations are often obtained through noisy measurements, simulations, or learned estimates, leading to \emph{bounded but potentially worst-case } errors. Such perturbations can destroy structural properties such as submodularity, monotonicity, or linearity, thereby invalidating classical approximation analyses that rely on exact oracles. This difficulty is particularly acute in bi-criteria settings, where errors in the objective and constraint can interact and compound. 
As a result, it is unclear how to even define, let alone analyze, offline approximation guarantees that are stable enough to support online learning objectives such as regret minimization and constraint violation control under bandit feedback. This raises a fundamental and largely unresolved question:
\begin{quote}
\emph{How can bi-criteria approximation algorithms be made robust to bounded oracle noise in a way that enables meaningful online regret and constraint violation guarantees?}
\end{quote}

In this work, we address this question by identifying the right offline abstraction needed to support online learning: \emph{resilience}. While resilience and robustness notions have been studied in single-objective combinatorial optimization and bandit settings, extending them to bi-criteria problems is fundamentally nontrivial due to the coupled degradation of objective and constraint guarantees under noise. We introduce a formal notion of $(\alpha,\beta,\delta,\texttt{N})$-\emph{resilience} for bi-criteria approximation algorithms, which characterizes how joint approximation guarantees degrade under bounded (possibly worst-case) oracle perturbations. This notion provides an \((\alpha, \beta)\)-bi-criteria guarantee with \(\delta\)-resilience (i.e., tolerance to perturbations in function oracles) and makes \(\texttt{N}\) oracle calls. Crucially, this definition makes no stochastic assumptions on the oracle errors and requires guarantees to hold uniformly over all oracle calls, making it strong enough to serve as a foundation for subsequent online regret and constraint violation guarantees.


Our main contribution is a \emph{general black-box offline-to-online framework} that leverages resilience to derive online learning guarantees. We introduce a new class of problems, \emph{Bi-Criteria Combinatorial Multi-Armed Bandits (BC-CMAB)} with bandit feedback, where the learner repeatedly selects combinatorial actions and observes only aggregate noisy reward and cost feedback. The notion of resilience captures robustness of offline bi-criteria approximation algorithms to worst-case, bounded oracle perturbations, while the online reduction operates in a stochastic bandit setting where repeated plays enable concentration. We show that any offline $(\alpha,\beta,\delta,\texttt{N})$-resilient bi-criteria approximation algorithm can be transformed into an online algorithm achieving sublinear regret and sublinear cumulative constraint violation (CCV). The resulting guarantees scale as $\mathcal{O}(\delta^{2/3}\texttt{N}^{1/3}T^{2/3}\log^{1/3}T)$ and require no structural assumptions such as linearity, monotonicity, submodularity, or semi-bandit feedback on the noisy functions.

As a second contribution, we establish resilience guarantees for several classical greedy algorithms in bi-criteria combinatorial optimization. In particular, we show that algorithms for Submodular Cover (SC), Submodular Cost Submodular Cover (SCSC), and Fair Submodular Maximization (FSM) satisfy $(\alpha,\beta,\delta,\texttt{N})$-resilience under additive oracle perturbations. Proving these results requires new analysis techniques, as standard greedy proofs rely critically on exact submodularity and monotonicity. These resilience guarantees serve as concrete instantiations of the framework and immediately yield online regret and constraint violation bounds under bandit feedback. While the problems studied in the paper use submodular functions, for the framework 
we do not assume linearity, submodularity, or problem-specific structures on noisy functions, enabling broad applicability.

Notably, even in the single-objective setting with bandit feedback, without exploiting special reward structures such as linearity, the best known regret bounds avoiding combinatorial dependence scale as $\tilde{\mathcal{O}}(T^{2/3})$ \cite{nie2023framework}. Moreover, $\Omega(T^{2/3})$ lower bounds are known for submodular maximization under bandit feedback when comparing against greedy benchmarks \cite{tajdini2024nearly}. Our results therefore extend the frontier of offline-to-online reductions to the bi-criteria regime, providing the first general theoretical foundation for learning with competing objectives under noisy oracle access. Unlike the single-objective case, noise simultaneously degrades both objective and constraint guarantees, and errors can propagate across criteria in a coupled manner. As a result, stability analyses based on single-objective resilience do not directly apply.

A core challenge lies in analyzing how bi-criteria guarantees degrade under  noisy oracles with bounded perturbations that violate structural assumptions. 
For instance, the greedy analysis for SC critically relies on submodularity and exact marginal gains \cite{goyal2013minimizing}, properties no longer preserved under noise. To address this, we develop new proof techniques: (1) We generalize density bounds (Lemma \ref{lemma:density-bound:exact}) for inexact oracles, ensuring greedy selections approximate the optimal sequence despite non-submodular perturbations. (2) We resolve recursive cost bounds with additive error propagation by identifying a logarithmic inequality (Lemma \ref{lemma:log-ineq:MINTSS:bound}), enabling us to recover approximation ratios with graceful noise-dependent terms. (3) For termination, we prove the algorithm halts even with noisy constraints by bounding the interaction between error tolerance ($\epsilon$) and problem parameters (e.g., $\omega$, $c_{\min}/c_{\max}$), ensuring feasibility. 
We note that the proof of resilience is not straightforward;
to illustrate this we note that where \cite{goyal2013minimizing} claims to prove resilience with multiplicative error, the proof has a mistake in generalizing density bounds (See Appendix \ref{counter-goyal}).

The key contributions of this work are as follows.

\begin{enumerate}[leftmargin=*,topsep=0pt]

\item \textbf{Black-box offline-to-online framework.}
We develop a general black-box framework that converts discrete bi-criteria offline approximation algorithms into online algorithms for \emph{Bi-Criteria Combinatorial Multi-Armed Bandits (BC-CMAB)} with bandit feedback. The framework requires no structural assumptions, and applies to a broad class of bi-criteria combinatorial optimization problems.

\item \textbf{Resilience as the enabling offline abstraction.}
We formalize a notion of $(\alpha,\beta,\delta,\texttt{N})$-\emph{resilience} for bi-criteria offline approximation algorithms, which characterizes how joint objective and constraint guarantees degrade under bounded (possibly worst-case) oracle noise. This notion is designed to be strong enough to support online learning guarantees while remaining agnostic to the source of noise.

\item \textbf{General online regret and constraint violation guarantees.}
We show that if an offline bi-criteria approximation algorithm satisfies $(\alpha,\beta,\delta,\texttt{N})$-resilience, then our framework produces an online BC-CMAB algorithm achieving
\[
\mathcal{O}\!\left(\delta^{2/3}\texttt{N}^{1/3}T^{2/3}\log^{1/3}T\right)
\]
regret and cumulative constraint violation under bandit feedback. The guarantees depend only on the resilience parameters of the offline algorithm.

\item \textbf{Resilience of classical bi-criteria algorithms.}
As concrete instantiations of the framework, we establish $(\alpha,\beta,\delta,\texttt{N})$-resilience for several classical greedy algorithms in bi-criteria combinatorial optimization, including Submodular Cover~\citep{goyal2013minimizing}, Submodular Cost Submodular Cover~\citep{crawford2019submodular}, and Fair Submodular Maximization~\citep{chen2024fair}. These results require new noise-aware analyses that go beyond existing exact-oracle proofs.

\item \textbf{Online consequences for canonical problems.}
Combining the framework with the above resilience guarantees immediately yields the first sublinear regret and cumulative constraint violation guarantees for these bi-criteria problems under bandit feedback. A summary of the resulting bounds is provided in Table~\ref{tab:summaryofresults}.

\end{enumerate}

\noindent

Overall, the proposed framework decouples online learning guarantees from problem-specific structure by isolating resilience as the key offline requirement, with classical bi-criteria optimization problems serving as representative instantiations.


\section{Related Work}\label{sec:rw}
\begin{table*}[t]
\caption{Summary of the $(\alpha, \beta, \delta, \texttt{N})$-resilient approximation for bi-criteria problems, including Submodular Cover (SC), Submodular Cost Submodular Cover (SCSC), and Fair Submodular Maximization (FSM), with the corresponding regret guarantees in CMAB under bandit feedback. This work establishes the first sublinear regret with cumulative constraint violation (CCV) under bandit feedback. 
Here, $\alpha$ and $\beta$ denote approximation factors for the objective $f$ and constraint $g$, respectively, 
where monotonicity (\texttt{Mon}) and submodularity (\texttt{Sub}) correspond to the properties of objectives and constraints. $\delta$ quantifies resilience to approximation, and $\texttt{N}$ represents the number of oracle calls to the offline algorithm. 
$h\triangleq \max(f_{\max},g_{\max})$. Details of problem-dependent parameters and other notations are discussed in Section \ref{sec:app}.}
    \label{tab:summaryofresults}
    \centering
    \resizebox{\textwidth}{!}{
    \begin{tabular}{c c c c c | c c c}
        \hline
        App. & Objective $f$ & $g$ & $\alpha$ & $\beta$ & $\delta$ & $\texttt{N}$ & Our Regret \& CCV \\
        \hline
        SC & $\min$\texttt{(Linear)} & \texttt{Mon}+\texttt{Sub} & $1+\ln\frac{{\kappa}}{\omega}$ & $1-\frac{\omega}{\kappa}$ & 
        $\frac{c_{\max}}{\omega c_{\min}}f_{\max}(3 + 6n)$
        & $n^2$ & $\mathcal{O}\left( 
        n^{4/3} f_{\max}^{5/3} T^{2/3} \log^{1/3}(T)
        \right)$ \\  
        SCSC & $\min (\texttt{Mon+Sub})$ & \texttt{Mon+Sub} & $\rho\left(\ln \left(\frac{\Psi}{\gamma}\right)+2\right)$ & $1$ & $\max\left\{\frac{8 c_{\max}}{c_{\min}\mu} \rho\left(\ln \left(\frac{\Psi}{\gamma}\right)+2\right)f_{\max}, 1\right\}$ & $n^2$ & $\mathcal{O}\left( n^{4/3} h^{5/3} T^{2/3} \log^{1/3}(T)\right)$ \\
        FSM & $\max (\texttt{Mon+Sub})$ & \texttt{Mon} & $\frac{1}{1+\omega}$
        &$ \frac{1}{\omega}$&$\max\left\{\frac{4\kappa}{1+\omega},1\right\}$&$\frac{n\kappa}{\omega}$ & $\mathcal{O}\Big( n^{1/3} f_{\max} T^{2/3} \log^{1/3}(T)\Big)$ \\
        \hline
    \end{tabular}}
\end{table*}

\subsection{Offline Bi-Criteria Optimization with Combinatorial Set Selection}
Offline bi-criteria optimization has been extensively studied for combinatorial problems such as submodular cover, fair submodular maximization, and knapsack-constrained optimization. 
Key problems include minimizing a submodular cost function while ensuring a utility threshold \cite{wolsey1982analysis,wan2010greedy,goyal2013minimizing,crawford2019submodular}, maximizing submodular utility under fairness constraints \cite{chen2024fair}, 
and balancing budget adherence with objective guarantees \cite{iyer2013submodular}. 
Most works assumed exact oracles or linear rewards, limiting applicability to online settings with bandit feedback. 
\citet{crawford2019submodular} 
considered the offline submodular cost submodular cover problem by analyzing the approximate guarantee of a specific greedy algorithm using an inexact oracle. 
\citet{goyal2013minimizing} analyzed a greedy algorithm for the submodular cover under a 
class of inexact oracles (multiplicative noise) for social network influence propagation. These works do not address online learning or bandit feedback and do not provide guarantees under noisy oracle perturbations (additive bounded noise).


\subsection{CMAB with Semi-bandit Feedback}

There are several works that focus on CMAB under semi-bandit feedback.
Some of these include proposing a general CMAB framework under a linear reward setting \cite{pmlr-v28-chen13a} and an extension 
to a non-linear reward setting \cite{NIPS2016_aa169b49}.
A few works have also looked into Thompson sampling based algorithms to provide a general framework to solve CMAB under semi-bandit feedback~\cite{kong2021hardness,wang2018thompson}. Adversarial CMAB with non-linear rewards is considered in \cite{han21b}.
Recent advances address constrained CMAB but remain limited to linear structures. For example, \citet{NEURIPS2022_13f17f74} consider linear rewards under linear constraints; \citet{pmlr-v202-li23aw} studies best-arm identification with knapsack constraints; and \citet{cmab-fair} consider the linear fairness constraints in CMAB under semi-bandit feedback. 
Additionally, 
\citet{lin2015stochastic,yu2016linear,takemori2020submodular} analyze marginal gains as feedback under the semi-bandit setting, enabling the learner to maximize rewards with multiple constraints based on individual gains. 
Notably, no prior semi-bandit framework supports \textit{bi-criteria} optimization with non-linear or combinatorial constraints.

\subsection{Single Objective CMAB with Bandit Feedback}

CMAB with bandit feedback has been widely investigated \cite{agarwal2021dart,agarwal2022stochastic,nie2022explore,nie2023framework,nie2024stochastic,fourati2023randomized,fourati2024combinatorial,fouratifederated,tajdini2024nearly}. 
Prior general frameworks for CMAB under bandit feedback, such as those by \cite{nie2023framework,fouratifederated}, convert offline algorithms into online algorithms using $(\alpha,\delta)-$resilience, but they focus solely on single-objective optimization. These works assume that the offline algorithm tolerates noisy reward estimates but does not address multiple objectives. Most general bandit frameworks require resilience or robustness conditions to handle noisy function estimates. 
Additional works include \cite{niazadeh2021online}, which presents an offline-to-online transformation for single-objective optimization under adversarial bandit feedback using the Blackwell approachability. \citet{streeter2008online} provide a related approach for the limited adversarial feedback in the context of single-objective submodular maximization.
For instance, semi-bandit methods impose structural assumptions like monotonicity and smoothness \cite{pmlr-v28-chen13a}, while bandit frameworks (e.g., \cite{nie2023framework}) rely on $\delta$-resilience to ensure approximation guarantees degrade gracefully with estimation errors. 
We note that \citet{pmlr-v28-chen13a} considered offline ``$(\alpha,\beta)$ approximation oracles,''
where $\alpha$ was an approximation coefficient for the objective, but $\beta$ referred to a success probability (that the solution satisfied the  $\alpha$ approximation guarantee).




Our framework addresses this gap by 
extending the offline algorithms to an online setting 
and requiring only that offline algorithms 
(e.g., algorithms for 
submodular cover
\cite{wolsey1982analysis,goyal2013minimizing},
submodular cost submodular cover \cite{wan2010greedy},
and
fair submodular maximization \cite{chen2024fair}) 
satisfy 
$\delta$-resilience (Definition~\ref{def:robustness}), a property we show holds for several existing bi-criteria approximation algorithms. 
This allows seamless conversion of offline guarantees to online BC-CMAB with sublinear regret and constraint violations, without problem-specific structures.

\section{Resilience of Offline Algorithms}  

This section formalizes the resilience property required for offline bi-criteria approximation algorithms. 
Resilience ensures that small errors in evaluating the objective and constraint functions 
during real world use
do not 
catastrophically degrade  the  performance.
This property will also enable black-box conversion of offline algorithms into online bi-criteria CMAB algorithms with sublinear regret and constraint violation bounds.

\begin{definition}[$(\alpha, \beta, \delta, \texttt{N})$-Resilient Approximation]\label{def:robustness}  
An offline algorithm \(\mathcal{A}\) is an \((\alpha, \beta, \delta, \texttt{N})\)-resilient approximation algorithm for the bi-criteria problem  
\[
\max_{S \subseteq \Omega} 
f(S) \quad \text{subject to} \quad g(S) \leq \kappa,
\]  
if, given access to approximate deterministic oracles 
$\hat{f}$ and $\hat{g}$
satisfying \(|f(S) - \hat{f}(S)| < \epsilon\) and \(|g(S) - \hat{g}(S)| < \epsilon\) for all \(S \subseteq \Omega\), 
\(\mathcal{A}\) returns a solution \(S^\mathcal{A}\) such that:  
\begin{align}
    \mathbb{E}[f(S^\mathcal{A})] \geq \alpha f(\mathrm{OPT}) - \delta \epsilon
    \quad
    \text{and}
    \quad
    \mathbb{E}[g(S^\mathcal{A})] \leq \beta \kappa + \delta \epsilon,
\label{eq:def:resilience}
\end{align}  
where \(\mathrm{OPT} = \arg\max_{S \subseteq \Omega} f(S) \text{ s.t. } g(S) \leq \kappa\). Here, \(\texttt{N}\) bounds the total number of oracle calls to \(\hat{f}\) and \(\hat{g}\), and \(\delta\) quantifies resilience to approximation errors.   The expectation is over the algorithm’s randomness; the oracle perturbations are otherwise arbitrary and need not be stochastic.
\end{definition}  

The key role of $(\alpha, \beta, \delta, \texttt{N})$-resilient approximation is that it is exactly the condition required for the offline-to-online conversion in Section \ref{sec:problem}. Any weaker notion (e.g.,  multiplicative robustness or stochastic noise assumptions) would not work for such an offline-to-online conversion under bandit feedback.

When defining the resilience property on functions $f$ and $g$, one could use different parameters $\delta_f$ and $\delta_g$ where $\mathbb{E}[f(S^\mathcal{A})] \geq \alpha f(\mathrm{OPT}) - \delta_f \epsilon$ and 
$\mathbb{E}[g(S^\mathcal{A})] \leq \beta \kappa + \delta_g \epsilon$. 
However, for the sake of simplicity, we use $\delta = \max\{\delta_f, \delta_g\}$.  
Also, for simple exposition, \cref{def:robustness} is defined for combinatorial bi-criteria problems (specifically over a power set).  

We note that when \(g\) (likewise \(f\)) is deterministic and known (i.e., \(\hat{g}(S) = g(S)\)), the resilience condition on \(g\), \(|g(S) - \hat{g}(S)| < \epsilon\) is not needed. Thus, we will not use/need the resilience condition on the function that is deterministic, while both 
inequalities in \eqref{eq:def:resilience}
will remain the same. 
  

The \(\delta\)-resilience term ensures that small errors (\(\epsilon\)) in estimating \(f\) or/and \(g\) (e.g., due to noisy bandit feedback) do not compound arbitrarily. 
This stability is essential for real-world use, where some function values need to be approximated or estimated. 
 It will also be important to extend offline algorithms to the online setting, where function estimates are inherently imperfect. 

The definition above is for maximization problems, but can be directly extended to 
minimization problems 
$\min_{S \subseteq \Omega} f(S) \ \text{subject to} \ g(S) \geq \kappa$.
See Appendix \ref{apd_res_def} for details. 

Bi-criteria algorithms under noisy feedback, such as   \cite{crawford2019submodular} for SCSC, demonstrate resilience to inexact oracles but remain confined to offline settings. {As mentioned in \cite{crawford2019submodular}, if the value oracle function satisfies some nice properties such as monotonicity and submodularity, then approximation guarantees can be achieved using the existing results for SCSC. However, it is not always necessary that the value oracle function satisfies these structural properties. We extend the offline algorithm of \cite{crawford2019submodular} to online BC-CMAB by formally defining the resilience guarantee and extending the approximation guarantees in \cite{crawford2019submodular}  to resilience guarantees as a direct corollary. We further provide resilience guarantees for two more fundamental problems, SC and FSM, under any value oracle function (not necessarily monotone and submodular). We next adapt these resilience guarantees for developing an online BC-CMAB algorithm with sublinear regret guarantees.}

\section{Resilience Guarantees for Different Problems}\label{sec:app}
In this section, we briefly discuss resilience 
of some 
of the 
bi-criteria approximation offline algorithms solving the problems of  Submodular Set Cover (SC), Submodular Cost Submodular Cover (SCSC), and Fair Submodular Maximization (FSM) problems. 
These problems cover a wide variety of applications in social influence maximization
\cite{goyal2013minimizing,han2017cost},
recommendation systems 
\cite{el2011beyond,guillory2011simultaneous},
active set selection
\cite{norouzi2016efficient}.
Further, these are just some examples where we provide the resilience guarantees. 

\subsection{Resilience Guarantee for Submodular Cover Problem}  
The Submodular Cover (SC) problem is a minimization problem where the goal is to find a subset \(S \subseteq \Omega\) that:  (i) minimizes a 
linear non-negative cost \(f(S) = \sum_{x \in S} c_x\),
and
(ii) satisfies \(g(S) \geq \kappa\), where \(g\) is a submodular utility function. 
%
%
\citet{goyal2013minimizing} proposed  a bi-criteria approximation algorithm \textsc{GREEDY-MINTSS} for this problem,
which achieves the following guarantees provided access to exact oracles for $f$ and $g$. The algorithm is provided in Algorithm \ref{alg:mintss}. 

\begin{algorithm}[t]
\caption{\textsc{Greedy-MINTSS} \cite{goyal2013minimizing}}
\label{alg:mintss}
\begin{algorithmic}[1]
\REQUIRE Ground set $\Omega$, utility function $g: 2^\Omega \to \mathbb{R}^+$, cost function $f(S) = \sum_{i \in S} c_i$, threshold $\kappa$, tolerance $\omega$.
\STATE Initialize $S \leftarrow \emptyset$
\WHILE{${g}(S) < \kappa - \omega$}
    \STATE Find $i^* \in \arg\max_{i \in \Omega \setminus S} \frac{\min({g}(S \cup \{i\}), \kappa) - {g}(S)}{c_i}$ \hfill \textit{\color{gray}// Maximize marginal utility per unit cost}
    \STATE Add $i^*$ to $S$: $S \leftarrow S \cup \{i^*\}$
\ENDWHILE
\RETURN $S$
\end{algorithmic}
\end{algorithm}

\begin{lemma}[Bi-Criteria Guarantees of \textsc{Greedy-MINTSS}, \cite{goyal2013minimizing}]  
\label{thm:MINTSS:offline}  
For any \(\omega > 0\), 
using exact oracles, 
\textsc{Greedy-MINTSS} 
outputs a solution $S$ satisfying
$f(S) \leq \alpha \cdot f(\mathrm{OPT})$ and $g(S) \geq \beta \cdot \kappa$ where $\alpha = 1 + \ln(\kappa/\omega)$ and $\beta = 1 - \omega/\kappa.$ 
\end{lemma}



{In this work, we focus on the problem where $g$ is stochastic.  There are several applications, such as online portfolio design, database query processing, sensor placement, and viral marketing, that requires online variant of submodular cover problem.  For example, in the application of viral marketing the influence $g(S)$ provided by a set of nodes  $S$ is generally not known beforehand and can only be obtained by sponsoring these set of influencers (nodes). Typically, the cost incurred to each sponsor $x$ is fixed and hence known beforehand, therefore we assume that $f$ is known in this setting. }
%
%
%
%
%
%
%
We show that when an inexact oracle \(\hat{g}\) is used instead of \(g\), 
with \(|\hat{g}(S) - g(S)| \leq \epsilon\), \textsc{GREEDY-MINTSS} exhibits \(\delta\)-resilience. 
The cost guarantee degrades by an additive \(\delta \epsilon\), while the utility guarantee is relaxed multiplicatively by \(\beta\) and additively by $\epsilon$. 
Since this is a minimization problem, we use the resilience definition in Appendix \ref{apd_res_def}.
In the following statement, we denote \(c_{\max} = \max_{x \in \Omega} c_x\), \(c_{\min} = \min_{x \in \Omega} c_x\), and \(n = |\Omega|\). 

\begin{theorem}
\label{thm:MINTSS:robust}  
For any $\omega > 0$, \textsc{Greedy-MINTSS} \citep{goyal2013minimizing} is an 
\((\alpha, \beta, \delta, \texttt{N})\)-resilient approximation algorithm for SC 
with $\epsilon \le \omega\frac{c_{\min}}{4 n c_{\max}}$, where:  
\[
\alpha = 1 + \ln(\kappa/\omega), \quad \beta = 1 - \omega/\kappa, 
\quad \delta = \frac{(3 + 6n)c_{\max}}{\omega c_{\min}}f_{\max}, 
\quad \texttt{N} = n^2.  
\]  
\end{theorem}

See Appendix~\ref{apd:mintss} for the proof.  
Since our resilience definition, \cref{def:robustness}, has approximation factors $\alpha$ and $\beta$ for the objective and constraint respectively, we seek to recover the same coefficients when inexact oracles are used. 
Thus, we aim to preserve the basic proof structure used in the exact oracle case \cite{goyal2013minimizing}. 
Developing new variants of algorithms that are designed to be resilient to inexact oracles is a valuable research direction but out of scope for this work.
We next briefly highlight some of the technical challenges encountered in the proof for inexact oracles.


We emphasize here that the existing bi-criteria proof of \textsc{Greedy-MINTSS} naturally uses the submodularity and monotonicity of the exact oracle $g$. 
However, the inexact oracle $\hat{g}$ 
 is neither submodular nor monotone. 
Every invocation of those properties in proof steps needs to be modified (if possible to do so).  
Another important consequence is that the sequence of sets $(S_1, S_2, \dots)$  chosen using an inexact oracle $\hat{g}$ could differ completely from the sequence chosen using an exact oracle $g$. In the following, we describe some of the key novelty in the proof.

\textbf{Generalizing Marginal Density Bounds (\cref{lemma:density-bound:exact}):}
In the exact setting, a key component of the  analysis is a marginal density bound that guarantees progress in each greedy iteration. However, this lemma assumes submodularity. In our case, since $\hat{g}$
 is neither submodular nor monotone, we cannot apply this directly. More precisely, a key lemma lower-bounding marginal densities with respect to a utility gap, \cref{lemma:density-bound:exact}, only holds for the exact oracle $g$ and asserts a property about the density of the element $x'$ with the largest density with respect to the exact oracle $g$.  
(Also in \cite{goyal2013minimizing}, the lemma was also only shown for the specific sequence of sets chosen by \textsc{GREEDY-MINTSS} with an exact oracle $g$.)  
Where that marginal density bound is used, we  provide  lower bound on marginal gains by the element with the largest density under the inexact oracle $\hat{g}$.  We also note that  for an inexact oracle, \cite{goyal2013minimizing} stated (without proof) a variant of their Lemma 1 for the threshold gap with respect to the inexact oracle for the unit cost (e.g., cardinality) case with multiplicative error. However, we note that this generalization is invalid, by showing a counterexample in Appendix \ref{counter-goyal}. Thus, we note that the procedure of establishing resilience is subtle; we show that a commonly cited generalization in \cite{goyal2013minimizing} does not hold under arbitrary inexact oracles.



\textbf{Logarithmic Inequality and Cost Bound Recovery:} The original proof of \cite{goyal2013minimizing} uses a clean recursive inequality involving exponentials to relate the cost of the final set to the optimum. However, when oracle noise is present, this recursion gains an additive error term that propagates and makes unraveling non-trivial. A major hurdle in generalizing the proof 
is identifying a recursive cost bound that can be used to relate the cost $f(S_{\ell-1})$ at the second to last \texttt{while} loop to the cost of the optimal solution $f(\mathrm{OPT})$ (see proof parts 3 and 4).  
For the utility gap $\kappa_i$ (w.r.t. threshold $\kappa$), 
in the case of exact oracles in \cite{goyal2013minimizing}, we have a recursion of the form $\kappa_i \leq \kappa_{i-1} \exp(-c_{x_i}/f(\mathrm{OPT}))$.  
A simple unraveling over $\ell-1$ iterations and using that $f$ is additive yields $\kappa_{\ell-1} \leq \kappa \exp( - f(S_{\ell-1})/f(\mathrm{OPT}))$.  
Then using that $\omega < \kappa_{\ell-1}$ we can easily get the bound $f(S_{\ell-1}) \leq f(\mathrm{OPT})\ln(\frac{\kappa}{\omega})$. 

However, working with inexact oracles we need to account for errors, and our bound has an additive error term 
$\kappa_i \le \kappa_{i-1}\exp({\frac{-c_{x_i}}{f(\mathrm{OPT})}}) + \frac{3\epsilon c_{\max}}{c_{\min}}$.  
Unraveling yields mixed multiplicative--additive terms.  Observing that the exponential term is bounded by one, we bound the terms to keep just a single additive term with $\epsilon$,
$\kappa_{\ell-1} \leq 
\kappa \exp({ - \frac{f(S_{\ell-1})}{f(\mathrm{OPT})}})+ 3\epsilon \frac{c_{\max}}{c_{\min}}\ell$.

With this formula, we can attempt to proceed like the exact oracle case and with rearranging could obtain $f(S_{\ell-1}) \leq f(\mathrm{OPT})
\ln( \kappa/(\omega - 3\epsilon \frac{c_{\max}}{c_{\min}}\ell ) ) $.  
However, since $\epsilon$ is in a logarithm with $\kappa$, 
when we later bound the cost of the final set $f(S_\ell)$ (see proof step 5), 
we would not be able to rearrange terms to  recover the approximation guarantee with 
a coefficient $\alpha = 1+\ln(\kappa/\omega)$ and an additive $\epsilon\delta$ term.

To recover a bound of the form
$f(S_\ell) \leq (1+\ln \frac{\kappa}{\omega}) + \epsilon\delta $, we need our bound on $f(S_{\ell-1})$ to have a similar form, $f(S_{\ell-1})\leq \ln \frac{\kappa}{\omega} + \epsilon(\dots))$.  
For this we identified and proved \cref{lemma:log-ineq:MINTSS:bound} which allows us to obtain a bound on $f(S_{\ell-1})$ in that form.

\textbf{Guaranteeing Termination under Noisy Constraints:}
We note that even the question of whether the algorithm reaches the stopping criteria becomes non-trivial when the criteria is based on an inexact oracle $\hat{g}$.
The sequence of sets $\{S_1, S_2, \dots\}$ chosen using an inexact oracle could significantly differ from the sequence chosen using an exact oracle.  
If $\hat{g}(S)<g(S)$ for all sets $S$, for instance, it is possible that even though a set passed the threshold for stopping with respect to $g$, that is $g(S_i)\geq \kappa - \omega$ for some index $i$,  with respect to $\hat{g}$ it has not ($\hat{g} < \kappa - \omega$) and the while loop could continue on.
We show that provided $\epsilon$ is not too large with respect to problem parameters, 
\textsc{GREEDY-MINTSS} will still reach the stopping criteria (for non-trivial thresholds $k$) even using the inexact $\hat{g}$.  
The range of $\epsilon \in [0, \omega\frac{c_{\min}}{4 n c_{\max}}]$ is based on  later analysis. 
\black

\subsection{Resilience Guarantee for Submodular Cost Submodular Cover Problem}

The Submodular Cost Submodular Cover (SCSC) problem \cite{wan2010greedy} involves finding a subset $S \subseteq \Omega$ that minimizes a submodular cost function $f(S)$ while ensuring that the utility of the selected set, captured by another submodular function $g(S)$, satisfies a lower bound $\kappa$. Formally, the problem can be expressed as:
\[
\text{Minimize } f(S) \quad \text{subject to} \quad g(S) \geq \kappa.
\]
In this problem, both the objective $f(S)$ and the constraint $g(S)$ are submodular functions. 
To address this challenge, 
\citet{wan2010greedy} proposed and analyzed the \textsc{Greedy-SCSC} algorithm (Algorithm~\ref{alg:greedySCSC} in \cref{apdx:GreedySCSC}) assuming exact oracles.
\citet{crawford2019submodular} 
analyzed \textsc{Greedy-SCSC} using an inexact oracle  $\hat{g}$. 

\begin{lemma}%
[Bi-Criteria Guarantees of \textsc{Greedy-SCSC} \cite{crawford2019submodular}]
\label{lem:SCSC:grd:crawford2019submodular}
The \textsc{Greedy-SCSC} Algorithm, when run with $\epsilon$-approximate oracle $\hat{g}$ (i.e., a fixed oracle $\hat g$ satisfying $\lvert g(S) - \hat g(S) \rvert \le \epsilon  \text{ for all } S$), returns a subset $S$ satisfying:
\[
f(S) \leq \frac{\rho}{1 - \frac{4\epsilon c_{\max}}{c_{\min}\mu}} \left( \ln \left( \frac{\Psi}{\gamma} \right) + 2 \right) f(OPT)
\qquad \text{and} \qquad
g(S) \geq \kappa - \epsilon,
\]
where $c_{\min} = \min_{x \in \Omega} f(\{x\})$, $c_{\max} = \max_{x \in \Omega} f(\{x\})  $, $\Psi = \max_{x \in \Omega} g(\{x\})$,
 $\gamma = \min \{\min\{g(A_i \cup \{x\}) - g(A_i),\kappa\}: {i \in [|\Omega|], x \in \Omega} \}$,
 $\mu = \min \{g(A_i) - g(A_{i-1})\}$, where $A_i$ represents the set selected at the $i$-th iteration,
  $\rho = \max_{X \subseteq \Omega} \frac{\sum_{x \in X} f(x)}{f(X)}$ denotes the curvature of the submodular function $f$,
 and it is assumed that $\mu > \frac{4\epsilon c_{\max}}{c_{\min}}$.

\end{lemma}

\cref{lem:SCSC:grd:crawford2019submodular}
nonlinearly combines the cost function approximation guarantee $\alpha$ (when an exact oracle is available) with the oracle error term $\epsilon$.
We can decouple those terms to obtain $\delta$-resilience guarantees.
%
To simplify the cost bound, 
let $ \frac{4\epsilon c_{\max}}{\mu c_{\min}} \le \frac{1}{2}$. 
Using the inequality $(1 - y)^{-1} \leq 1 + 2y$ for $y \le \frac{1}{2}$, we obtain:
\[
f(S) \leq \left( 1 + 2\epsilon \frac{4 c_{\max}}{c_{\min} \mu} \right) \rho \left( \ln \left( \frac{\Psi}{\gamma} \right) + 2 \right) f(OPT).
\]

This directly implies the following  $\delta$-resilience guarantees:

\begin{corollary}
The \textsc{Greedy-SCSC} Algorithm is an $(\alpha, \beta, \delta, \texttt{N})$-resilient approximation algorithm for the monotone Submodular Cost Submodular Cover problem, when \(\frac{4\epsilon c_{\max}}{\mu c_{\min}} \le \frac{1}{2}\),  where:
\[
\alpha = \rho \left( \ln \left( \frac{\Psi}{\gamma} \right) + 2 \right), 
\quad \beta = 1, \]
%
\[\delta = \max \left\{ \frac{8 c_{\max}}{c_{\min} \mu} \cdot \rho \left( \ln \left( \frac{\Psi}{\gamma} \right) + 2 \right) f_{\max},1 \right\}, 
\quad \texttt{N} = n^2.
\]
\end{corollary}




\subsection{Resilience Guarantee for Fair Submodular Maximization}
Fair Submodular Maximization  (FSM)  is different from the previous two problems in two ways. 
First, it is a maximization problem of submodular function under cardinality constraints. Secondly, this problem has an additional fairness constraint, which requires that the selected set must contain the necessary fraction of elements from each group. 
More formally, the base set $\Omega$ is partitioned into $C$ groups represented by $\{\Omega_c\}_{c=1}^C$. The Fair Submodular Maximization problem aims to maximize a monotone submodular function $f(S)$ under cardinality and group fairness constraints. Formally, 
\begin{align}
    \text{Maximize }  
     f(S) \  
    \text{ s.t. }  l_c \le |S \cap \Omega_c| \leq u_c \ \forall c \in [C],  
     |S| \leq \kappa,
\end{align}
where $u_c$ and $l_c$ are the upper and lower bounds for group $c$, $\kappa$ is the cardinality constraint.
\citet{chen2024fair} proposed a bi-criteria algorithm \textsc{Greedy-Fairness-Bi} (see \cref{alg:fairness-bi} in Appendix \ref{apd_fsm}) for this problem.
We first state their bi-criteria guarantee, with $\beta \geq 1$ relaxing the fairness constraint:

\begin{lemma}[Bi-Criteria Guarantees of \textsc{Greedy-Fairness-Bi}, \cite{chen2024fair}] For any $\omega \in (0,1]$ such that $1/\omega\in \mathbb{N}_+$, 
using exact oracles
\textsc{Greedy-Fairness-Bi} returns a subset $S$ satisfying:
\begin{align*}
&f(S) \ge \alpha\cdot f(OPT), \text{where}\; \alpha = \frac{1}{1+\omega},\\ 
&|S \cap \Omega_c| \leq \beta u_c \quad \forall c \in [C],\\
&\sum_{c\in C} \max\{|S \cap \Omega_c|, \beta l_c\} \leq \beta \kappa, \text{where}\; \beta = \frac{1}{\omega}
\end{align*}
Further, $\frac{n}{\omega \kappa}$ bounds the number of queries.
\end{lemma}

\begin{theorem}
\label{thm:fsm-resilience}
The \textsc{greedy-fairness-bi} algorithm achieves an $(\alpha, \beta, \delta, \texttt{N})$-resilient bi-criteria approximation for FSM with: $\alpha = \frac{1}{1+\omega}$, $\beta = \frac{1}{\omega}$, $\delta =  \frac{4\kappa}{1 + \omega}$, and $\texttt{N} = \frac{n\kappa}{\omega},$
where $\omega \in (0,1)$ controls the approximation-constraint trade-off, $\kappa$ is the cardinality constraint, and $n = |\Omega|$. 
\end{theorem}
The proof 
is provided in Appendix \ref{apd_fsm}. 

\section{From Resilient Offline Algorithms to Online Bandit Guarantees}\label{sec:problem}

We next study sequential combinatorial decision-making over a finite horizon \( T \). Let \( \Omega \) be a ground set of \( n \) base arms and at each time step \( t \), the learner selects an action \( A_t \subseteq \Omega \) and observes a stochastic reward \( f_t(A_t) \in [0, f_{\max}] \) and a cost \( g_t(A_t) \in [0, g_{\max}] \), both drawn from unknown distributions  (assumed independent across time for each fixed action) with expectations \( f(A) = \mathbb{E}[f_t(A)] \) and \( g(A) = \mathbb{E}[g_t(A)] \).
 The learner’s goal is to maximize the cumulative reward \( \sum_{t=1}^T f_t(A_t) \) while ensuring that the expected cost of each action approximately satisfies a constraint 
\( \kappa \in (0,1) \).
Formally, we require:  
\(
\frac{1}{T}\sum_{t=1}^T g_t(A_t) \leq  \kappa.
\)

 For the offline resilience definition (Definition \ref{def:robustness}), we assume bounded oracle perturbations. In contrast, the online  model studied in this section assumes stochastic bandit feedback: for each fixed action, repeated plays yield independent, bounded observations with fixed expectations, enabling concentration of empirical means. The regret and cumulative constraint violation guarantees in this section rely on this stochastic feedback model and do not extend to adversarial or adaptive bandit noise.

We note that  our framework can also handle minimization problems subject to a lower bound on the utility function (see Appendix \ref{app:min} for more details). However, for easy exposition, the framework is explained with the help of the maximization function subject to an upper bound constraint.

Since directly optimizing  \( f \) over a constraint on \(g\) is generally NP-hard, for example, maximizing a submodular function under knapsack constraints, comparing to an exact oracle is impractical unless \( T \) is exponentially large. 
Instead, it may be more natural to compare against what is achievable (in polynomial time) by offline approximation algorithms.  Some such cases have an 
\((\alpha, \beta)\)-bi-criteria approximation algorithm \( \mathcal{A} \), where \( \alpha \in (0, 1] \) and \( \beta \geq 1 \). We define the reward regret and the cumulative constraint violation (CCV) in terms of such approximations as follows. 

Let \( \mathrm{OPT} \) denote the optimal action with respect to the expected objective and constraint functions
\[
\mathrm{OPT} \in \arg\max_{A \subseteq \Omega} f(A) \quad \text{subject to} \quad g(A) \leq \kappa.
\] 
The regret is defined as the gap between \( \alpha \)-scaled cumulative reward of the optimal feasible action and the learner’s reward. More formally, 
\[
\mathbb{E}[\mathcal{R}_f(T)] = \alpha T f(\mathrm{OPT}) - \mathbb{E}\left[\sum_{t=1}^T f_t(A_t)\right], 
\]  
The cumulative constraint violation (CCV)  measures how much the learner’s cumulative cost exceeds the relaxed budget \( \beta T \kappa \), and is formally defined as  
\[
\mathbb{E}[\mathcal{V}_g(T)] = \mathbb{E}\left[\sum_{t=1}^T g_t(A_t)\right] - \beta T \kappa. 
\]
In our setting, the learner receives bandit feedback: after selecting action \( A_t \), the learner observes only the reward \( f_t(A_t) \) and cost \( g_t(A_t) \) associated with \( A_t \), with no information about other actions.
We are assuming that \( f_t \) and \( g_t \) are stochastic---drawn from an unknown distribution with mean \( f(A_t) \) and \( g(A_t) \), respectively. As a special case, this also includes the cases where one of \( f_t \) or \( g_t \) is deterministic (i.e., \( f_t(A) = f(A) \) for all \( t \) or \( g_t(A) = g(A) \) for all \( t \)). For instance, in budgeted recommendation systems, costs (e.g., monetary expenses) might be fixed and known a priori, whereas rewards (e.g., user engagement) are stochastic. However, even in such cases, the learner must still balance exploration-exploitation trade-offs for the other stochastic  function. 
 Our framework naturally accommodates both scenarios: it handles noisy \( f_t \) (or \(g_t\)) (where \( f(A_t) \) (or \(g(A_t)\)) is observed with randomness) and deterministic \( f \) (or \(g\)).


\subsection{Algorithm Description}
\begin{algorithm}[H]
\caption{\textsc{Bi-Criteria CMAB Algorithm}}\label{alg:bi-criteria-algo}
\begin{algorithmic}[1]
\REQUIRE Horizon $T$, ground set $\Omega$, 
$(\alpha,\beta,\delta,\texttt{N})$-resilient algorithm $\mathcal{A}$.
\STATE Set $m \gets \left\lceil \frac{\delta^{2/3} T^{2/3} (\log T)^{1/3} }{2\texttt{N}^{2/3}}\right\rceil$
\STATE \textbf{Exploration Phase:}
\WHILE{$\mathcal{A}$ queries action $A$}
    \FOR{$j=1$ to $m$}
        \STATE Play $A$, observe $f^{(j)}(A)$, $g^{(j)}(A)$ 
    \ENDFOR
    \STATE Compute $\bar{f}(A) = \frac{1}{m}\sum_{j=1}^m f^{(j)}(A)$
    \STATE Compute $\bar{g}(A) = \frac{1}{m}\sum_{j=1}^m g^{(j)}(A)$ 
    \STATE Return $\bar{f}(A)$, $\bar{g}(A)$ to $\mathcal{A}$
\ENDWHILE
\STATE \textbf{Exploitation Phase:}
\STATE Let $S \gets$ output of $\mathcal{A}$
\WHILE{$t\leq T$}
    \STATE Play $S$
\ENDWHILE
\end{algorithmic}
\end{algorithm}

Our framework, \textsc{Bi-Criteria CMAB Algorithm} (Algorithm~\ref{alg:bi-criteria-algo}), converts an offline $(\alpha, \beta, \delta, \texttt{N})$-resilient bi-criteria approximation algorithm $\mathcal{A}$ into an online CMAB algorithm. It operates in two phases:

\begin{enumerate}[leftmargin=*]
    \item \textbf{Exploration Phase:} For each subset $A \subseteq \Omega$ queried by $\mathcal{A}$, play $A$ for $m$ rounds. In round $j$, observe noisy realizations $f^{(j)}(A)$ and $g^{(j)}(A)$ of the underlying objective and constraint functions. Define the empirical estimates
\[
\bar f(A) \triangleq \frac{1}{m}\sum_{j=1}^m f^{(j)}(A),
\qquad
\bar g(A) \triangleq \frac{1}{m}\sum_{j=1}^m g^{(j)}(A).
\]
Return $\bar f(A)$ and $\bar g(A)$ to $\mathcal{A}$ as inexact oracle evaluations.

    \item \textbf{Exploitation Phase:} Deploy $\mathcal{A}$'s output action $S$ for all remaining rounds.
\end{enumerate}

\subsection{Regret and CCV Analysis}

Our framework ensures sublinear regret for the reward objective \(f\)
 and sublinear cumulative constraint violation (CCV) for the cost constraint \(g\). The theorem below formalizes these guarantees, demonstrating that our algorithm adapts offline resilience to handle online uncertainty while balancing exploration and exploitation.

\begin{theorem}[Regret and CCV Guarantees]\label{thm:main}  
For a bi-criteria CMAB instance that admits an  $(\alpha, \beta, \delta, N)-$resilient approximate offline algorithm $\mathcal{A}$,
\textsc{Bi-Criteria CMAB Algorithm} run with $\mathcal{A}$ for a horizon \( T \geq \max\left\{\texttt{N}, \frac{2\sqrt{2}\texttt{N}}{\delta }\right\} \) achieves the following $\alpha$-regret and CCV, where $h\triangleq \max(f_{\max},g_{\max})$:  
\[
\mathbb{E}[\mathcal{R}_f(T)] = \mathbb{E}[\mathcal{V}_g(T)] = \mathcal{O}\left(\delta^{2/3}h \texttt{N}^{1/3}T^{2/3}\log^{1/3}T\right).
\]
\end{theorem}

\begin{remark}
 This result represents the first bi-criteria optimization result for CMAB.
 Notably, it does not exploit the problem structure and avoids any combinatorial dependence on the number of arms. 
 Additionally, \citet{tajdini2024nearly} established that for monotone stochastic submodular bandits with a cardinality constraint, a regret scaling of \(\mathcal{O}(T^{2/3})\) is unavoidable when compared to the greedy algorithm, provided that combinatorial dependence on the arms is avoided---a necessity for small to moderate \(T\). 
\end{remark}
\begin{proof}[Proof Sketch] 
We highlight a few key steps here.  See Appendix~\ref{apd:main} for the full proof.
Let $\mathcal{E}$ denote the clean event where all empirical mean estimates satisfy 
$|\bar{f}(A_i) - f(A_i)| < \text{rad}$ and $|\bar{g}(A_i) - g(A_i)| < \text{rad}$, with 
$\text{rad} = \sqrt{\frac{h^2 \log T}{2m}}$. 
Under $\mathcal{E}$, we decompose (conditional) $\alpha$-regret and CCV  into separate terms for exploration and exploitation phases:
\begin{align*}
\mathbb{E}[\mathcal{R}_f(T) | \mathcal{E}] &=
\sum_{i=1}^N m \left(\alpha f(\text{OPT}) - \mathbb{E}[f(S_i)] \right) 
+ (T - Nm) \left(\alpha f(\text{OPT}) - \mathbb{E}[f(S)] \right),\\
\mathbb{E}[\mathcal{V}_g(T) | \mathcal{E}] &=\sum_{i=1}^N m \left(\mathbb{E}[g(S_i)] - \beta \kappa \right) + (T - Nm) \left(\mathbb{E}[g(S)] - \beta \kappa \right).
\end{align*}
%
%
%

We bound exploration phase $\alpha$-regret and CCV using $f(\text{OPT}) \leq h$ and $g(S_i) \leq h$ respectively.

During the exploitation phase, (under $\mathcal{E}$) the $\delta$-resilience property ensures:
\begin{equation*}
\mathbb{E}[f(S)] \geq \alpha f(\text{OPT}) - \delta \cdot \text{rad} \quad \text{and} \quad \mathbb{E}[g(S)] \leq \beta \kappa + \delta \cdot \text{rad}.
\end{equation*}
We can then bound the exploitation phase $\alpha$-regret and CCV respectively as:
\begin{align*}
(T - Nm) \left(\alpha f(\text{OPT}) - \mathbb{E}[f(S)] \right) &\leq T \delta \cdot \text{rad}, \\
(T - Nm) \left(\mathbb{E}[g(S)] - \beta \kappa \right) &\leq T \delta \cdot \text{rad}.
\end{align*}
%
%
We optimize $m$ as
\begin{equation*}
m = \Theta\left( \frac{\delta^{2/3} T^{2/3} (\log T)^{1/3}}{N^{2/3}} \right)
\end{equation*}
to minimize total $\alpha$-regret and CCV, which leads to the result.

\end{proof}
\color{black}

\begin{remark}
    Definition \ref{def:robustness}  requires the approximation guarantees to hold uniformly over all subsets \(S \subseteq \Omega\), even though the online reduction in Section~5 only ever queries the offline algorithm on a finite (data-dependent) collection of sets. This stronger, uniform formulation is intentional. It allows resilience to serve as a purely offline, black-box abstraction that is independent of the internal query pattern or adaptivity of the algorithm when deployed online. In particular, the online reduction only uses the resilience guarantees on the sets actually queried, but uniform resilience ensures that these guarantees hold without requiring the online learner to reason about or restrict the offline algorithm’s behavior. We note that weaker notions of resilience restricted to queried sets would suffice for a fixed algorithmic instantiation, but would entangle the offline property with the online execution and reduce composability.

\end{remark}

We note that this result can be combined with the three studied applications in the previous section. Further, the application-specific theorems impose a lower bound on the horizon $T$ through the choice of the accuracy parameter $\epsilon$, which is instantiated as the confidence radius $\mathrm{rad}$ in the online algorithm.
 The results are given as follows.


\begin{corollary}
    For the Submodular Cover problem, the \textsc{Bi-Criteria CMAB Algorithm} achieves the following regret and CCV bounds. For \( T \geq \max\left\{n^2, 
    \frac{2 \sqrt{2} n^2 \omega c_{\min}}{f_{\max}(3+6n)}
    \right\} 
    \) and \( \frac{T}{\log T} \geq \frac{64 \texttt{N} n^3 c_{\max}^3 f_{\max}^3}{\delta \omega^3 c_{\min}^3}\):
\[
\mathbb{E}[\mathcal{R}_f(T)] = \mathbb{E}[\mathcal{V}_g(T)]  = \mathcal{O}\left( 
n^{4/3} f_{\max}^{5/3} T^{2/3} \log^{1/3}(T)
\right),
\]
where $h=f_{\max}\le n c_{\max}$. 
\end{corollary}


\begin{corollary}
    For the monotone Submodular Cost Submodular Cover problem, the \textsc{Bi-Criteria CMAB Algorithm} achieves the following regret and CCV bounds. For  \( T \geq \max\left\{\texttt{N}, \frac{2\sqrt{2}\texttt{N}}{\delta}\right\} \) and \( \frac{T}{\log T} \geq \frac{512 \texttt{N} c_{\max}^3}{\delta \mu c_{\min}^3}\):
\begin{align}
\mathbb{E}[\mathcal{R}_f(T)] = \mathbb{E}[\mathcal{V}_g(T)]  = \mathcal{O}\Bigg( 
n^{4/3} h^{5/3}  T^{2/3}\log^{1/3}(T)
\Bigg).
\end{align}
\end{corollary}


\begin{corollary}
    For the Fair Submodular Maximization problem, the \textsc{Bi-Criteria CMAB Algorithm} achieves the following regret and CCV bounds. For \( T \geq \frac{n}{\omega}\max\left\{\kappa, 1+\omega\right\} \):
\[
\mathbb{E}[\mathcal{R}_f(T)] = \mathbb{E}[\mathcal{V}_g(T)] =\mathcal{O}\left( n^{1/3} f_{\max} T^{2/3} \log^{1/3}(T)\right).
\]
\end{corollary}


\section{Conclusions}\label{sec:conc}


We developed a general black-box framework for bi-criteria combinatorial optimization under noisy function evaluations. The framework is based on a notion of $(\alpha,\beta,\delta,\texttt{N})$-resilience, which characterizes how joint objective and constraint guarantees degrade under bounded oracle perturbations, and serves as the sole offline requirement for deriving online guarantees. Using this abstraction, we showed how resilient offline approximation algorithms can be converted into online algorithms with bandit feedback, achieving sublinear regret and cumulative constraint violation without relying on structural assumptions on the noisy functions.

As instantiations, we established resilience for classical greedy algorithms in Submodular Cover, Submodular Cost Submodular Cover, and Fair Submodular Maximization, yielding the first general bi-criteria regret guarantees under bandit feedback. An important direction for future work is extending resilience and offline-to-online guarantees to broader objective classes and adversarial online feedback models.

\section{Acknowledgement}
This work is supported in part by  the U.S. National Science Foundation under grants CCF-2149588 and CCF-2149617. 

\bibliography{Reference}
\bibliographystyle{tmlr}

    \newpage

\onecolumn


\appendix






\addcontentsline{toc}{section}{Appendices} 
\renewcommand*\contentsname{Table of Contents}
\addtocontents{toc}{\protect\setcounter{tocdepth}{1}} 
\tableofcontents 

\clearpage

\clearpage

\section{Resilience Definition for Minimization}\label{apd_res_def}


The offline algorithm \( \mathcal{A} \) is an \( (\alpha, \beta, \delta, \texttt{N}) \)-resilient approximation for:  
\[
\text{Minimize } f(S) \quad \text{subject to} \quad g(S) \geq \kappa, \quad S \subseteq \Omega,
\]
if, given noisy oracles \( \hat{f}, \hat{g} \) with \( |f(S) - \hat{f}(S)| < \epsilon \) and \( |g(S) - \hat{g}(S)| < \epsilon \), it returns \( S^\mathcal{A} \) such that:  
\begin{align}
\mathbb{E}[f(S^\mathcal{A})] &\leq \alpha f(\mathrm{OPT}) + \delta \epsilon, \label{eq:resilience-f-min} \\
\mathbb{E}[g(S^\mathcal{A})] &\geq \beta \kappa - \delta \epsilon. \label{eq:resilience-g-min}
\end{align}

We also note that in this case $\alpha\ge 1$ and $\beta\le 1$.

\section{Algorithm and Proof for the Submodular Cover Problem}\label{apd:mintss}

\noindent \textbf{Algorithm Setup:}
\begin{itemize}
    \item \textbf{Input}: Ground set \(\Omega\), deterministic cost \(f(S) = \sum_{x \in S} c_x\),  utility oracle (exact $g$ or inexact \(\hat{g}\)), threshold \(\kappa\), parameter \(\omega > 0\).
    \item \textbf{Goal}: Minimize \(f(S)\) subject to \(g(S) \geq \kappa\).
    \item \textbf{Resilience Conditions}: For \(\epsilon > 0\),
    \[
    \mathbb{E}[f(S)] \leq \alpha f(\mathrm{OPT}) + \delta \epsilon, \quad 
    \mathbb{E}[g(S)] \geq \beta \kappa - \delta \epsilon.
    \]
\end{itemize}

The detailed offline algorithm for the problem is given in Algorithm \ref{alg:mintss}, which was proposed in \cite{goyal2013minimizing}.

We first generalize a result from \cite{goyal2013minimizing} 
that will be used in our analysis. 
Denote the cost function as $f(S) = \sum_{x\in S} c_x$, where $c_x$ represents the cost of the base arm $x$, which we assume is known and is not stochastic.


\begin{remark}
    Lemma 1 in \cite{goyal2013minimizing} was shown for the specific sets chosen by the greedy algorithm using an  exact value oracle ($g$).  
    We show essentially the same proof holds for any set $S \subset \mathcal{X}$, which will be critical for our analysis when an exact value oracle is unavailable and the sequence of subsets chosen by the algorithm using $\hat{g}$ may be completely different from the sequence of subsets that would have been chosen using $g$.
\end{remark}

\begin{lemma}\label{lemma:density-bound:exact}
For a non-negative, monotone non-decreasing submodular set function $g:\mathcal{X} \to \mathbb{R}^{\geq 0}$ and positive monotone cost function 
$f:\mathcal{X} \to \mathbb{R}_+$,
for any set $S \subset \mathcal{X}$, there is an element $x \in \mathcal{X} \backslash S$ such that
\begin{align}
    \frac{\min(g(S \cup \{x\}),\kappa) - \min(g(S),\kappa)}{c_x} \geq \frac{\kappa - \min(g(S),\kappa)}{f(\mathrm{OPT})} ,
    \label{eq:lemma:density-bound:exact}
\end{align}
where $\mathrm{OPT}$ is the minimal cost set satisfying $g(\mathrm{OPT})\geq \kappa$.
\end{lemma}


\begin{proof}  
The proof of \cref{lemma:density-bound:exact} essentially follows along the lines of Lemma 1 in \cite{goyal2013minimizing}.  
As noted in \cite{goyal2013minimizing}, thresholded monotone submodular functions, such as $\min(g(\cdot),\kappa)$ are also monotone and submodular.
    
If $g(S)\geq \kappa$, then the right hand side of \eqref{eq:lemma:density-bound:exact} is zero.  Since $g$ is monotone non-decreasing the left hand side is always non-negative, so the inequality trivially holds for any $x \in \mathcal{X} \backslash S$.

For $g(S) < \kappa$, we prove the lemma by contradiction.  
We will assume that for all elements $x \in \mathcal{X} \backslash S$,
    \begin{align}
    \frac{\min(g(S \cup \{x\}),\kappa) - \min(g(S),\kappa)}{c_x} < \frac{\kappa - \min(g(S),\kappa)}{f(\mathrm{OPT})}.
    \label{eq:lemma:density-bound:exact:contra}
    \end{align}

    This condition means $g(S)<\kappa$.
    We (arbitrarily) enumerate elements in the optimal set that are not in $S$, 
    \begin{align*}
        \{y_1, \dots, y_t\} = \mathrm{OPT}\backslash S,
    \end{align*}
    where $t$ is the number of such elements.  
    Since $g(S) < \kappa \leq g(\mathrm{OPT})$ we must have $t\geq1$ (at least one element).
    By monotonicity, $\kappa\leq g(\mathrm{OPT}) \leq g(S \cup \mathrm{OPT})$.
    We have
    %
    \begin{align}
        & \hspace{-1cm} \kappa - g(S) \nonumber\\
        &= \min( g(S \cup \mathrm{OPT}),\kappa) - \min(g(S),\kappa) \tag{$g(\mathrm{OPT})\geq \kappa$ by def; $g(S)<\kappa$ by assumption}\\
        %
        %
        &= \sum_{i=1}^t \min(g(S \cup \{y_1, \dots, y_i\}),\kappa) - \min(g(S \cup \{y_1, \dots, y_{i-1}\}),\kappa)  \tag{telescoping sum}\\
        &\leq \sum_{i=1}^t 
        \min(g(S \cup \{y_i\}),\kappa)  - \min(g(S ),\kappa)   \tag{submodularity of $\min(g(\cdot),\kappa)$}\\
        &< \sum_{i=1}^t c_{y_i} \frac{\kappa - \min(g(S),\kappa)}{f(\mathrm{OPT})}   \tag{using  assumption \Eqref{eq:lemma:density-bound:exact:contra}}\\
        &= f(\mathrm{OPT}\backslash S) \frac{\kappa - \min(g(S),\kappa)}{f(\mathrm{OPT})}  \nonumber\\
        &< \kappa - \min(g(S),\kappa)    \tag{$0<f(\mathrm{OPT}\backslash S) < f(\mathrm{OPT})$} \\
        &= \kappa - g(S), \nonumber
    \end{align}
    a contradiction.


    
\end{proof}



We  will also later use a logarithmic inequality. 


\begin{lemma} \label{lemma:log-ineq:MINTSS:bound}
    For $a,b\in \mathbb{R}_+$ such that $\frac{b}{a} \leq 0.79$, $ \ln(a-b) \geq \ln(a) - \frac{2b}{a}.  $
\end{lemma}

\begin{proof}
First, 
\begin{align*}
 \ln(a-b) = \ln(a(1-\frac{b}{a})) = \ln (a) + \ln (1- \frac{b}{a}).
\end{align*}

It suffices to check that $h(x) := \ln(1-x) + 2x\geq 0$ for $0\leq x\leq 0.79$.  We can confirm $h(x)$ is concave with two roots.

$h'(x) = \frac{-1}{1-x} + 2$ so $h'(0) = 1$ and there is a stationary point at $x=\frac{1}{2}$.

$h''(x) = \frac{-1}{(1-x)^2}$, so $h(x)$ is concave, increasing for $x<\frac{1}{2}$ and then decreasing for $x>\frac{1}{2}$.

$h(0) = 0$ trivially.  $h(0.79) \approx 0.01935225$ and $h(0.8) \approx -0.00943791$.  Thus one root is $x=0$ and the other root is in the interval $(0.79,0.8)$.  
    
\end{proof}

\color{black}

In the following, we will show the resilience guarantee of this algorithm, \cref{thm:MINTSS:robust}.

\begin{proof}
\noindent The proof follows along the following steps:

First, we note that even though inexact values (based on $\hat{g}(\cdot)$) are used, the algorithm will terminate for non-trivial values of the threshold $\kappa$ (i.e. $\kappa< g(\Omega)$).  
For any $S\subseteq \Omega$ with $g(S) \geq \kappa$ (including $ \mathrm{OPT}$), 
\begin{align*}
    \hat{g}(S) 
    &\geq g(S) - \epsilon \nonumber\\
    &\geq \kappa - \epsilon \nonumber\\
    &\geq \kappa -  \omega\frac{c_{\min}}{4 n  c_{\max}} \nonumber\\
    &\geq \kappa -  \omega. \nonumber\\
\end{align*}



\noindent \textbf{1. Noisy Utility Propagation:}
The algorithm terminates when \(\hat{g}(S) \geq \kappa - \omega\). Given \(|\hat{g}(S) - g(S)| \leq \epsilon\),  
\[
g(S) \geq \hat{g}(S) - \epsilon \geq (\kappa - \omega) - \epsilon.
\]
Rewriting for \(\beta\):
\[
g(S) \geq \left(1 - \frac{\omega}{\kappa}\right)\kappa - \epsilon = \beta \kappa - \epsilon.
\]
Thus, \(\beta = 1 - \frac{\omega}{\kappa}\) and the utility error term is \(\delta_g \epsilon = \epsilon \implies \delta_g = 1\).


\noindent \textbf{2. Cost Error Analysis:}
Let \(\mathrm{OPT} = \arg\min_{S' \subseteq \Omega} \{f(S') \mid g(S') \geq \kappa\}\). Let us denote $x_1, x_2, \ldots, x_\ell$ to be the elements added (in order) by the algorithm.
Define the set $S_i = \{x_1, x_2, \ldots, x_i\}$. 
Thus, $S_\ell$ denotes the final set outputed by the algorithm.  \textit{We explicitly set } $S_0=\emptyset$.
We want to bound $f(S_\ell)$. 

We first make two basic observations.  For  $i\leq\ell$, $\hat{g}(S_{i-1}) < \kappa$ since
the algorithm had  not yet stopped (and thus $\hat{g}(S_{i-1}) < \kappa-\omega$). 
We also have that 
$g(S_{i-1}) < \kappa$, since 
\begin{align}
    g(S_{i-1}) &\leq \hat{g}(S_{i-1}) + \epsilon \nonumber\\
    &\leq \kappa-\omega + \epsilon 
     \nonumber\\
    &\leq \kappa-\omega\left( 1 - \frac{c_{\min}}{4 n  c_{\max}}\right) 
     \nonumber\\
     &< \kappa. \nonumber
\end{align}
%
%
%
At each iteration, the algorithm selects \(x_i\) maximizing the noisy marginal density:  
\[x_i \gets \argmax_{x\in \Omega \backslash S_{i-1}} \hat{\rho}_{x}(S_{i-1}) = \frac{\min(\hat{g}(S_{i-1} \cup \{x\}),\kappa) - \hat{g}(S_{i-1})}{c_{x}}.
\]

Let $x_i'$ denote the element with largest marginal density (with respect to the true function $g$).

\[
x_i' \gets \argmax_{x\in \Omega \backslash S_{i-1}} \rho_{x}(S_{i-1}) = \frac{\min(g(S_{i-1} \cup \{x\}),\kappa) - g(S_{i-1})}{c_{x}}.
\]

Let us further denote $\tilde{g}(S) = \min(g(S),\kappa)$. Then, 
by Lemma \ref{lemma:density-bound:exact} (Density Bound), the largest true marginal gain satisfies


\begin{align}
 \frac{\tilde{g}(S_{i-1} \cup \{x_i'\}) - \tilde{g}(S_{i-1})}{c_{x_i'}} \geq \frac{\kappa - \tilde{g}(S_{i-1})}{f(\mathrm{OPT})}. \label{eq:prf:MINTSS:rob:densitybndbest}
\end{align}

We also have:
\begin{align*}
\hat{\rho}_{x_i}(S_{i-1}) &\ge \hat{\rho}_{x_i'}(S_{i-1}) \tag{greedy selection}\\ 
&\ge \rho_{x_i'}(S_{i-1}) - \frac{2\epsilon}{c_{x_i'}} \tag{value error bound}\\ 
&\ge \frac{\kappa - \tilde{g}(S_{i-1})}{f(\mathrm{OPT})} - \frac{2\epsilon}{c_{x_i'}} \tag{using \eqref{eq:prf:MINTSS:rob:densitybndbest}}\\ 
&\ge \frac{\kappa - \min(\hat{g}(S_{i-1}), \kappa)}{f(\mathrm{OPT})} - \frac{\epsilon}{f(\mathrm{OPT})}- \frac{2\epsilon}{c_{x_i'}} \tag{value error bound}\\
&\ge \frac{\kappa - \min(\hat{g}(S_{i-1}), \kappa)}{f(\mathrm{OPT})} - 3\frac{\epsilon}{c_{\min}} .
\end{align*}

\noindent \textbf{3. Recursive Cost Bound:}

\noindent\textbf{Base case (i=1):}
Let \(\kappa_0 := \kappa - \min\{\hat g(\emptyset), \kappa\}\).
Since \(\hat g(\emptyset) \ge 0\), we have \(\kappa_0 \le \kappa\); hence the recurrence holds for \(i=1\).

Let us define the utility gap $\kappa_i = \kappa - \min(\hat{g}(S_i), \kappa)$. Then from the above inequality we get:

\noindent\textbf{General case ($i>1$):}

\begin{align*}
\hat{\rho}_{x_i}(S_{i-1})
= \frac{\kappa_{i-1} - \kappa_i}{c_{x_i}} &\ge \frac{\kappa - \min(\hat{g}(S_{i-1}), \kappa)}{f(\mathrm{OPT})} - 3\frac{\epsilon}{c_{\min}}\\
&= \frac{\kappa_{i-1}}{f(\mathrm{OPT})} - 3\frac{\epsilon}{c_{\min}}\\
\implies \kappa_i &\le \kappa_{i-1}\left(1-\frac{c_{x_i}}{f(\mathrm{OPT})}\right) + \frac{3\epsilon c_{\max}}{c_{\min}}\\
\implies \kappa_i &\le \kappa_{i-1}e^{\frac{-c_{x_i}}{f(\mathrm{OPT})}} + \frac{3\epsilon c_{\max}}{c_{\min}}
\end{align*}


\if \discussionSCalgchallenges 1
\blue
A major  hurdle is identifying a recursive cost bound (see proof parts 3 and 4) that can be used to relate the cost $f(S_{\ell-1})$ at the second-to-last step to the cost of the optimal solution $f(\mathrm{OPT})$. 

For the utility gap $\kappa_i$ (wrt main threshold $\kappa$), using exact oracles we have a recursion of the form $\kappa_i \leq \kappa_{i-1} \exp(-c_{x_i}/f(\mathrm{OPT}))$.  A simple unraveling over $\ell-1$ iterations and using that $f$ is additive yields $\kappa_{\ell-1} \leq \kappa \exp( - f(S_{\ell-1})/f(\mathrm{OPT}))$.  Then using that $\omega < \kappa_{\ell-1}$ we can easily get the bound $f(S_{\ell-1}) \leq f(\mathrm{OPT})\ln(\frac{\kappa}{\omega})$. 

However, working with inexact oracles we need to account for errors, and our bound has an additive error term $\kappa_i \le \kappa_{i-1}e^{\frac{-c_{x_i}}{f(\mathrm{OPT})}} + \frac{3\epsilon c_{\max}}{c_{\min}}$.  Unraveling yields many cross terms.  Observing that the exponential term is bounded by one, we bound the cross terms to keep just a single additive term with $\epsilon$,
$\kappa_{\ell-1} \leq \kappa e^{ - \frac{f(S_{\ell-1})}{f(\mathrm{OPT})}}+ 3\epsilon \frac{c_{\max}}{c_{\min}}\ell$.
\hl{this part maybe not too big} \red{fixed. taken value/iteration l as maximum number of "oracle calls" N at p=1, is 39.}

With this formula, we can attempt to proceed like the exact oracle case and with rearranging could obtain $f(S_{\ell-1}) \leq f(\mathrm{OPT})
    \ln\left( \frac{\kappa}{\omega - 3\epsilon \frac{c_{\max}}{c_{\min}}\ell } \right) $.  However, since $\epsilon$ in in a logarithmic term with $\kappa$, 
    when we later bound the cost of the final set $f(S_\ell)$ (see proof step 5), we would not be able to rearrange terms to  recover the approximation guarantee with an approximation coefficient $\alpha = 1+\ln(\kappa/\omega)$ with an additive $\epsilon\delta$ term.

To recover a bound of the form
$f(S_\ell) \leq (1+\ln \frac{\kappa}{\omega}) + \epsilon\delta $, we need our bound on $f(S_{\ell-1})$ to have a similar form, $f(S_{\ell-1})\leq \ln \frac{\kappa}{\omega} + \epsilon(\dots))$.  
For this we identified and proved \cref{lemma:log-ineq:MINTSS:bound} which allows us to obtain a bound on $f(S_{\ell-1})$ in that form.

\black

\fi

\noindent \textbf{4. Telescoping Sum:}  


Unrolling the recursion over \(\ell-1\) iterations, we obtain:
\begin{align}
    \kappa_{\ell-1} &\leq \kappa e^{ - \sum_{i=1}^{\ell-1}\frac{c_{x_i}}{f(\mathrm{OPT})}}+ 3\epsilon \frac{c_{\max}}{c_{\min}} (\ell-1) \nonumber\\
   &\leq \kappa e^{ - \frac{f(S_{\ell-1})}{f(\mathrm{OPT})}}+ 3\epsilon \frac{c_{\max}}{c_{\min}}\ell. \label{eq:l_1}
\end{align}

At termination it follows that $\kappa - \min({\hat{g}}(S_\ell),\kappa) \le \omega$  where $\omega$ is the threshold parameter. Also, because $\ell$ is the last iteration, we have $\kappa_{\ell-1} > \omega$ and $\kappa_\ell \le \omega$.

Thus,
\begin{align}
\omega &\leq \kappa_{\ell-1} \nonumber\\
&\leq \kappa e^{ - \frac{f(S_{\ell-1})}{f(\mathrm{OPT})}}+ 3\epsilon \frac{c_{\max}}{c_{\min}}\ell \tag{using \eqref{eq:l_1}} \nonumber\\
\implies \omega - 3\epsilon \frac{c_{\max}}{c_{\min}}\ell &\leq  \kappa e^{ - \frac{f(S_{\ell-1})}{f(\mathrm{OPT})}} \nonumber\\
\implies \ln(\omega - 3\epsilon \frac{c_{\max}}{c_{\min}}\ell) &\leq \ln(\kappa) - \frac{f(S_{\ell-1})}{f(\mathrm{OPT})} 
\label{eq:prf:MINTSS:30}
     \\
    \ln(\omega)-6\epsilon \frac{c_{\max}}{\omega c_{\min}} \ell&\leq \ln(\kappa) - \frac{f(S_{\ell-1})}{f(\mathrm{OPT})} \tag{\cref{lemma:log-ineq:MINTSS:bound}}
    \\
   \frac{f(S_{\ell-1})}{f(\mathrm{OPT})} &\leq \ln(\frac{\kappa}{\omega}) +6\epsilon \frac{c_{\max}}{\omega c_{\min}} \ell   \nonumber \\
   {f(S_{\ell-1})}&\leq f(\mathrm{OPT})\ln(\frac{\kappa}{\omega}) +f(\mathrm{OPT})6\epsilon \frac{c_{\max}}{\omega c_{\min}} \ell, \label{eq:c_l_1}
   \end{align}

\if \discussionSCalgchallenges 1
\blue  (what happens if try to follow exact proof strategy?)
\begin{align*}
    \omega &\leq \kappa_{\ell-1} \nonumber\\
    &\leq \kappa e^{ - \frac{f(S_{\ell-1})}{f(\mathrm{OPT})}}+ 3\epsilon \frac{c_{\max}}{c_{\min}}\ell \tag{using \eqref{eq:l_1}} \nonumber\\
    \implies
    \omega - 3\epsilon \frac{c_{\max}}{c_{\min}}\ell &\leq \kappa e^{ - \frac{f(S_{\ell-1})}{f(\mathrm{OPT})}} \nonumber\\
    \implies
    \ln\left( \frac{\omega - 3\epsilon \frac{c_{\max}}{c_{\min}}\ell }{\kappa} \right) &\leq  - \frac{f(S_{\ell-1})}{f(\mathrm{OPT})}\nonumber\\
    \implies
    f(S_{\ell-1}) &\leq f(\mathrm{OPT})
    \ln\left( \frac{\kappa}{\omega - 3\epsilon \frac{c_{\max}}{c_{\min}}\ell } \right) \nonumber\\
\end{align*}

Later on (after bound $c_{x_\ell}$ below, would get to
\begin{align*}
f(S_\ell) &= c_{x_\ell} + f(S_{\ell-1})\\
&\le f(\mathrm{OPT}) \left(1 + \frac{3\epsilon}{c_{\min}}\frac{c_{\max}}{\omega}\right) 
+ f(\mathrm{OPT})
    \ln\left( \frac{\kappa}{\omega - 3\epsilon \frac{c_{\max}}{c_{\min}}\ell } \right) 
\\
&?? f(\mathrm{OPT})\left(1 + \ln\left(\frac{\kappa}{\omega}\right)  \right) + \epsilon\dots
\end{align*}

\black
\fi
   
where for \eqref{eq:prf:MINTSS:30}, 
since $\epsilon \le \omega\frac{c_{\min}}{4 n  c_{\max}} $ and $\ell \leq n$, $\omega - 3\epsilon \frac{c_{\max}}{c_{\min}}\ell >0$.

Since the cost function is monotone,
\begin{align}\label{eqfub}
f(S_{\ell}) = c_{x_\ell} + f(S_{\ell-1}) 
\end{align}

where, $x_\ell = \arg\max_{x\in \Omega\setminus S_{\ell-1}}\hat{\rho}_x(S_{\ell-1})$. Further, let $x_\ell' = \arg\max_{x\in \Omega\setminus S_{\ell-1}}\rho_x(S_{\ell-1})$. Then,
\begin{align*}
\hat{\rho}_{x_\ell}(S_{\ell-1}) &\geq \hat{\rho}_{x_\ell'}(S_{\ell-1})\\
&\geq \rho_{x_\ell'}(S_{\ell-1}) - \frac{2\epsilon}{c_{x_\ell'}}\\
&\geq \frac{\kappa - \tilde{g}(S_{\ell-1})}{f(\mathrm{OPT})} - \frac{2\epsilon}{c_{x_\ell'}}\\
&\geq \frac{\kappa - \min(\hat{g}(S_{\ell-1}), \kappa)}{f(\mathrm{OPT})} - \frac{2\epsilon}{c_{x_\ell'}} - \frac{\epsilon}{f(\mathrm{OPT})}\\
&\geq \frac{\kappa - \min(\hat{g}(S_{\ell-1}), \kappa)}{f(\mathrm{OPT})} - \frac{3\epsilon}{c_{\min}} 
\end{align*}

Thus, we get,
\begin{align*}
\frac{\kappa_{\ell-1} - \kappa_\ell}{c_{x_\ell}} \geq \frac{\kappa_{\ell-1}}{f(\mathrm{OPT})} - \frac{3\epsilon}{c_{\min}}
\end{align*}
Rearranging,
\[\frac{\kappa_{\ell-1} - \kappa_\ell}{\kappa_{\ell-1}} \ge \frac{c_{x_\ell}}{f(\mathrm{OPT})} - \frac{3\epsilon}{c_{\min}}\frac{c_{x_\ell}}{\kappa_{\ell-1}}.\]

\if \discussionSCalgchallenges 1

    \todo[inline]{these steps seem valid, but looking at it now CJQ wonders if is necessary to invoke $\omega$ -- eg in MINTSS proof, their analogous argument worked for every element, that its cost was bounded by cost of optimal set}

\blue
If we try to use argument in MINTSS based on exact $g$, would apply lemma (here have additive error term), then like do below bound ratio by 1, but holds for all $i$ (in MINTSS paper), check if should hold here too.  The ratio is
\begin{align*}
    \frac{\kappa_{i-1} - \kappa_i}{\kappa_{i-1}} 
    &= \frac{(\kappa -\min(\hat{g}(S_{i-1}),\kappa) )- (\kappa -\min(\hat{g}(S_{i}),\kappa))}{\kappa - \min(\hat{g}(S_{i-1}),\kappa)} \\
    &= \frac{\min(\hat{g}(S_{i}),\kappa)- \min(\hat{g}(S_{i-1}),\kappa)  }{\kappa - \min(\hat{g}(S_{i-1}),\kappa)} \\
    &= \frac{\min(\hat{g}(S_{i}),\kappa)- \hat{g}(S_{i-1})  }{\kappa - \hat{g}(S_{i-1})} \tag{$i-1<\ell$ so the while condition is  true, so $\hat{g}(S_{i-1})<\kappa-\omega<\kappa$}\\
    &\leq \frac{\kappa- \hat{g}(S_{i-1})  }{\kappa - \hat{g}(S_{i-1})}\\
    &= 1
\end{align*}

\black
\fi
    
As noted above,  we have $\kappa_{\ell-1} > \omega$ and $\kappa_\ell \le \omega$.  Thus $\frac{\kappa_{\ell-1} - \kappa_\ell}{\kappa_{\ell-1}} \leq 1$.  Using that, we get


\[1 \geq \frac{c_{x_\ell}}{f(\mathrm{OPT})} - \frac{3\epsilon}{c_{\min}}\frac{c_{x_\ell}}{\kappa_{\ell-1}}\]

Rearranging and using the fact that $\kappa_{\ell-1} > \omega$,
\[\frac{c_{x_\ell}}{f(\mathrm{OPT})} \le 1 + \frac{3\epsilon}{c_{\min}}\frac{c_{x_\ell}}{\kappa_{\ell-1}} \le 1 + \frac{3\epsilon}{c_{\min}}\frac{c_{x_\ell}}{\omega}\]

Thus, we get 
\begin{align}
    c_{x_\ell} \le f(\mathrm{OPT}) \left(1 + \frac{3\epsilon}{c_{\min}}\frac{c_{\max}}{\omega}\right). \label{eq:MINTSSprf:cxlbound}
\end{align}

\if \discussionSCalgchallenges 1
\blue
To relate the cost of the final element $c_{x_\ell}$ to the cost of the optimal set $f(\mathrm{OPT})$, which is needed to bound the cost of the final set $f(S_\ell)$, with exact oracles, a simple application of \cref{lemma:density-bound:exact} can yields $c_{x_\ell} \leq f(\mathrm{OPT}$.  
That same argument actually holds for all elements in the greedy sequence (when exact oracle is used) as well  $c_{x_i} \leq f(\mathrm{OPT}$ for $i=1,\dots,\ell$.
However, \cref{lemma:density-bound:exact} only holds for the exact oracle $g$, not $\hat{g}$ (also in \cite{goyal2013minimizing} the lemma was also only shown for the specific sequence of sets chosen by MINTSS with an exact oracle $g$).
Carefully bounding errors, we can use \cref{lemma:density-bound:exact} for the final set $S_\ell$ and best element $x_\ell'$ with respect to the exact oracle $g$ to obtain a bound on the density of the final element $x_\ell$ actually chosen.
We can then show $\frac{\kappa_{\ell-1} - \kappa_\ell}{\kappa_{\ell-1}} \leq 1$.  However, in rearranging to get the ratio $\frac{\kappa_{\ell-1} - \kappa_\ell}{\kappa_{\ell-1}}$, the additive error term $-\frac{3\epsilon}{c_{\min}}\frac{c_{x_\ell}}{\kappa_{\ell-1}}$
has both the cost of the final element $c_{x_\ell}$  and the utility gap $\kappa_{\ell-1}$ for the second to last set. 
   
\black

\fi


\noindent \textbf{5. Resilience Parameter \(\delta\) and Oracle Calls \(\texttt{N}\):}

{

Combining \eqref{eq:c_l_1}, \eqref{eqfub}, and \eqref{eq:MINTSSprf:cxlbound},
\begin{align*}
f(S_\ell) &= c_{x_\ell} + f(S_{\ell-1})\\
&\le f(\mathrm{OPT}) \left(1 + \frac{3\epsilon}{c_{\min}}\frac{c_{\max}}{\omega}\right) + f(\mathrm{OPT})\ln(\frac{\kappa}{\omega}) +f(\mathrm{OPT})6\epsilon \frac{c_{\max}}{\omega c_{\min}} \ell\\
&= f(\mathrm{OPT})\left(1 + \ln\left(\frac{\kappa}{\omega}\right)  \right) + \epsilon\frac{c_{\max}}{\omega c_{\min}}f(\mathrm{OPT})(3 + 6\ell) \tag{rearranging}\\
&\leq f(\mathrm{OPT})\left(1 + \ln\left(\frac{\kappa}{\omega}\right)  \right) + \epsilon\frac{c_{\max}}{\omega c_{\min}}f_{\max}(3 + 6n). 
\end{align*}
Thus, we get $\alpha = 1 + \ln\left(\frac{\kappa}{\omega}\right)$ and $\delta = \frac{c_{\max}}{\omega c_{\min}}f_{\max}(3 + 6n)$.

}

Each iteration selects one element and there are at most \(n\)  elements. Further, each iteration queries \(g\) at most \(n\) times. Thus, \(\texttt{N} = \frac{1}{2}n(n+1)\) though for simplicity we use \(\texttt{N} = n^2.\)
\end{proof}

\color{black}

\section{Counter-example to \cite{goyal2013minimizing}'s 
generalization of their Lemma 1 for (multiplicative) inexact oracles for \textsc{GREEDY-MINTSS}}
\label{counter-goyal}

We note that the authors of \cite{goyal2013minimizing} claimed that it is `straightforward' to generalize their result to (multiplicative) inexact oracles, while we note in the following that such generalization 
does not hold. 
We point this out to highlight that care must be taken in extending results involving exact oracles to inexact oracles. 

Consider a modular (thus also monotone submodular) function constraint function $f$ (they use $f$ where we write $g$) with $f(1) = 0.99$, $f(2) = 0.01$, $f(i) = 0$ for $i=3,4,\dots$. Using \cite{goyal2013minimizing}'s notation, for the threshold $\eta = 1$, clearly for the unit-cost case (the only case explicitly considered in \cite{goyal2013minimizing} for inexact oracles), the optimal set is ${1,2}$ and optimal cost is cardinality $k=2$.

Consider a inexact oracle $f'$, with $f'(1)= 0.02$, $f'(2) = 0.01$, with other values equal to 0. The greedy algorithm using $f'$ would pick element 1 and then 2, which is also the sequence the greedy algorithm would pick using the exact oracle. \cite{goyal2013minimizing}'s claimed generalization of their Lemma 1 for this case after the first step has that there must exist an element $x \in \Omega \backslash {1}$ satisfying $$f({1,x}) - f({1}) \geq (\eta - f'({1}) )/k = (1 - 0.02)/2 = 0.98/2 = 0.47,$$ which is not true. They use this claim directly in their proof of resilience for multiplicative noise, and thus their resilience proof as written is not correct.

\section{\textsc{GREEDY-SCSC} Algorithm}
\label{apdx:GreedySCSC}
Here we include the pseudo-code for \textsc{Greedy-SCSC} by \cite{wan2010greedy}.  

    \begin{algorithm}[H]
\caption{\textsc{Greedy-SCSC}}\label{alg:greedySCSC}
\begin{algorithmic}[1]
    \STATE {\bfseries Require:} Submodular oracle 
    $g$,
    submodular cost function $f$ and threshold $\kappa$.
    \STATE Initialize $S \leftarrow \emptyset$.
    \WHILE {$
    g(S)
    < \kappa$}
        \STATE 
        $u \gets \argmax\limits_{i \in \Omega \setminus S} \frac{\min(
        g(S \cup \{i\}), 
        \kappa) - \min(
        g(S),
        \kappa)}{f(\{i\})}.$
        \STATE Update $S \leftarrow S \cup \{u\}$.
    \ENDWHILE
    \RETURN $S$.
\end{algorithmic}
\end{algorithm}


\section{Algorithm and Proof of Fair Submodular Maximization}\label{apd_fsm}

\begin{algorithm}[ht]
    \caption{\textsc{greedy-fairness-bi}}\label{alg:fairness-bi}
    \begin{algorithmic}[1]
        \STATE \textbf{Require: } $S\leftarrow\emptyset$, $\omega$, Partition set $P = \{\Omega_1, \Omega_2, \ldots, \Omega_C\}$, approximate fairness matroid $\mathcal{M}_{1/\omega} = \mathcal{M}_{1/\omega}(P,\kappa/\omega,\vec{l}/\omega,\vec{u}/\omega)
        $ 
        \STATE \textbf{Output: }$S\in \Omega$
        \WHILE{$\exists i$ s.t. $S\cup\{i\}\in\mathcal{M}_{1/\omega}$}
        \STATE $V \gets\{i\in \Omega \mid S \cup \{i\} \in \mathcal{M}_{1/\omega}\}$
        \STATE Find $i^* \in\arg\max_{i\in V\setminus S} \left( f(S\cup \{i\}) - f(S) \right)$
        \STATE  Add $i^*$ to $S$: $S \leftarrow S \cup \{i^*\}$
        \ENDWHILE
        \RETURN $S$
    \end{algorithmic}
\end{algorithm}

Before proving \cref{thm:fsm-resilience}, we first state a technical lemma from \cite{chen2024fair} relating the feasible regions of the problems with strict ($ 1/\omega = 1$) and relaxed ($ 1/\omega > 1$) fairness constraints, characterized as matroids  $\mathcal{M}_{1/\omega}(P,\kappa/\omega,\vec{l}/\omega,\vec{u}/\omega) = \{S \subseteq \Omega: |S \cap \Omega_c| \le \frac{u_c}{\omega}, \forall c\in[C], \sum_{c\in [C]}\max\{|S \cap \Omega_c|, \frac{l_c}{\omega}\} \le \frac{\kappa}{\omega}\}$.

\begin{lemma}[\cite{chen2024fair}]
\label{lem:fairsmc}
For any set $S \in \mathcal{M}_{1/\omega}(P, \kappa/\omega, \vec{l}/\omega, \vec{u}/\omega)$ with 
$|S| = \frac{\kappa}{\omega}$, $T \in \mathcal{M}_1(P, \kappa, \vec{l}, \vec{u})$, with $|T| = \kappa$, 
and any permutation of $S = (s_1,s_2,\ldots,s_{\kappa/\omega})$, 
there exists a sequence $E = (e_1,e_2,\ldots e_{\kappa/\omega})$ such that each element in $T$ appears $1/\omega$ times in $E$ and that $S_i \cup \{e_{i+1}\} \in \mathcal{M}_{1/\omega},\;\; \forall i\in \{0,1,\ldots,\kappa/\omega\}$ where 
$S_i = (s_1,s_2,\ldots,s_i)$ and 
$S_0 = \emptyset$.
\end{lemma}


\begin{proof}[Proof of \cref{thm:fsm-resilience}]
We begin by denoting the optimal solution of the problem as 
$
\mathrm{OPT} = \arg\max_{S \in \mathcal{M}_1(P, \kappa, \vec{l}, \vec{u})} f(S).
$ %
For iteration $i=1,\dots,\kappa/\omega$, let $S_i$ denote the set selected in that iteration.  
Apply  \cref{lem:fairsmc} with $ \mathrm{OPT} $ as $T$ to obtain a valid sequence $E$ of $1/\omega$ copies of $\mathrm{OPT}$. 
%


Since Algorithm \ref{alg:fairness-bi} stops after $\kappa/\omega$ steps and outputs a set $|S| = \kappa/\omega$ and $S \in \mathcal{M}_{1/\omega}$, we have $|S \cup \Omega_c| \le \frac{u_c}{\omega}\;\; \forall c\in[C]$ and $\sum_{c\in [C]}\max\{|S \cap \Omega_c|, l_c/\omega\} \le \kappa/\omega$.
These  two equations directly give resilience on cardinality as $\beta = 1/\omega$ and $\delta_g = 0$.

%
Consider the set $S_{i+1}$ chosen in the $(i+1)$th iteration.  Since the algorithm chose the element $S_{i+1}\backslash S_i$ instead of $e_{i+1}\in \mathrm{OPT}$ such that $S_i \cup \{e_{i+1}\} \in \mathcal{M}_{1/\omega}$,
\begin{align*}
f(S_{i+1}) - f(S_i) 
&\ge \hat{f}(S_{i+1}) - \hat{f}(S_i) - 2\epsilon\\
& \ge \hat{f}(S_i \cup \{e_{i+1}\}) - \hat{f}(S_i) -2\epsilon \\
&\ge f(S_i \cup \{e_{i+1}\}) - f(S_i) - 4\epsilon \\
&\ge f(S \cup \{e_{i+1}\}) - f(S) - 4\epsilon, 
\end{align*}
where the last step uses that $f$ is submodular.  
%
Sum both sides of the last inequality over all iterations $i=0,1,\dots,\kappa/\omega-1$.  
Now, $\sum_{i=0}^{\kappa/\omega - 1} f(S_{i+1}) - f(S_i) = f(S) - f(\emptyset) = f(S)$.
Also, from Lemma \ref{lem:fairsmc}, each $e_{i+1} \in E$, where $E$ is a sequence containing $1/\omega$ copies of each element in OPT.  
Therefore, 
$\sum_{i=0}^{\kappa/\omega - 1} f(S \cup \{e_{i+1}\}) - f(S) $
is equal to
$
1/\omega \sum_{i^*\in OPT} f(S \cup i^*) - f(S).  
$ 
%
%
Using a well-known identity for monotone submodular functions, 
\begin{align*}
\sum_{i^*\in \mathrm{OPT}} f(S \cup \{i^*\}) - f(S) 
&\geq f(\mathrm{OPT}) - f(S).
\end{align*}

Consequently, we get:
\begin{align*}
f(S) 
\ge 1/\omega \left[\sum_{i^*\in \mathrm{OPT}} f(S \cup \{i^*\}) - f(S)\right] - 4\epsilon\frac{\kappa}{\omega}
\ge \frac{f(\mathrm{OPT})-f(S)}{\omega} - 4\epsilon\frac{\kappa}{\omega}.
\end{align*}
Rearranging terms and 
%
observing that since the algorithm runs for $ \kappa/\omega$ steps, and uses at most $n$ oracle calls in each step, the total oracle calls are bounded by \(\texttt{N} = \frac{n\kappa}{\omega}\).
\end{proof}


\section{Clean Event Bound}\label{apdx:clean}


We define key events for our analysis. 
For each action \( A \) played during exploration, the \( m \) observed rewards are i.i.d. with mean \( f(A) \) and are bounded in \([0,f_\text{max}]\). 
Likewise, the \( m \) observed constraint values are i.i.d. with mean \( g(A) \) and are bounded in \([0,g_\text{max}]\).
For simplicity, we will use the bound $h = \max\{f_\text{max},g_\text{max} \}$
By Hoeffding’s inequality, the empirical means \( \bar{f}(A) \) and \( \bar{g}(A) \)  satisfy the respective concentration bounds  
\[
\mathbb{P} \left( \big| \bar{f}(A) - f(A) \big| \geq \epsilon \right) \leq 2 \exp \left( - \frac{2m \epsilon^2}{h^2} \right) 
\quad \text{and} \quad  
\mathbb{P} \left( \big| \bar{g}(A) - g(A) \big| \geq \epsilon \right) \leq 2 \exp \left( - \frac{2m \epsilon^2}{h^2} \right).
\]  
These bounds hold for all actions played during exploration.

%

\begin{lemma}[Concentration of Empirical Means in Exploration]
\label{lem:concentration}
Let \( A_1, \dots, A_\texttt{N} \) be the set of actions played during the exploration phase, each played \( m \) times. 
Suppose the rewards and constraints associated with these actions are bounded in \([0,h]\), and let \(\bar{f}(A_i)\) and \(\bar{g}(A_i)\) denote of action \( A_i \)'s empirical mean reward and the empirical mean constraint value respectively.
Define the confidence radius  
\[
\mathrm{rad} := \sqrt{\frac{h^2 \log T}{2m}}.
\]  
Then, with probability at least 
\( 1 - 4\texttt{N} T^{-1} \), the normalized empirical means of all actions remain within \(\mathrm{rad}\) of their true means, i.e.,  
\[
\mathcal{E} := \bigcap_{i=1}^{\texttt{N}} \left\{ \big| \bar{f}(A_i) - f(A_i) \big| < \mathrm{rad} \right\} \cap \left\{ \big| \bar{g}(A_i) - g(A_i) \big| < \mathrm{rad} \right\}
\]
holds.
\end{lemma}

\begin{proof}
Applying Hoeffding’s inequality to each action \( A_i \), we obtain  
\[
\mathbb{P} \left( \big| \bar{f}(A_i) - f(A_i) \big| \geq \mathrm{rad} \right) \leq 2 \exp \left( - \frac{2m \mathrm{rad}^2}{h^2} \right)
\] and \[ \mathbb{P} \left( \big| \bar{g}(A_i) - g(A_i) \big| \geq \mathrm{rad} \right) \leq 2 \exp \left( - \frac{2m \mathrm{rad}^2}{h^2} \right).
\]
Substituting \(\mathrm{rad} = \sqrt{h^2 \log(T)/2m}\), we compute  
\[
\mathbb{P} \left( \big| \bar{f}(A_i) - f(A_i) \big| \geq \mathrm{rad} \right) \leq 2 T^{-1}
\quad \text{and} \quad
\mathbb{P} \left( \big| \bar{g}(A_i) - g(A_i) \big| \geq \mathrm{rad} \right) \leq 2 T^{-1}.
\]

Define the event that the empirical means for both the objective and constraint functions are within the confidence bound for the $i$th action as 
\(\mathcal{E}_i = \{ |\bar{f}(A_i) - f(A_i)| < \mathrm{rad} \} \cap \{ |\bar{g}(A_i) - g(A_i)| < \mathrm{rad} \}\).

Denote the complement of $\mathcal{E}_i$ as
\[
\mathcal{E}_{i}^{c} = \{ |\bar{f}(A_i) - f(A_i)| \geq \mathrm{rad}   \} \cup \{ |\bar{g}(A_i) - g(A_i)| \geq \mathrm{rad}   \}.
\] 
By the union bound,
\begin{align*}
    \mathbb{P} \left( \mathcal{E}^{c}_i \right) 
    &\leq \mathbb{P} \left( \big| \bar{f}(A_i) - f(A_i) \big| \geq \mathrm{rad} \right)
    +
    \mathbb{P} \left( \big| \bar{g}(A_i) - g(A_i) \big| \geq \mathrm{rad} \right)  \\
    &\leq 4 T^{-1}.
\end{align*}


Let $\mathcal{E}^{c}$ denote the complement of the clean event $\mathcal{E}$. Using the union bound,  
\[
\mathbb{P} \left( \mathcal{E}^{c} \right) = \mathbb{P} \left( \bigcup_{i=1}^{\texttt{N}} \mathcal{E}_{i}^{c} \right) \leq \sum_{i=1}^{\texttt{N}} \mathbb{P} \left( \mathcal{E}_{i}^{c} \right) \leq \sum_{i=1}^{\texttt{N}} 4 T^{-1} = 4\texttt{N} T^{-1}.
\]  
Thus,  
\[
\mathbb{P}(\mathcal{E}) = 1 - \mathbb{P}(\mathcal{E}^{c}) \geq 1 - 4\texttt{N} T^{-1},
\]  
completing the proof.

\section{Proof of Theorem \ref{thm:main}}\label{apd:main}

\begin{proof} 
Let $\mathcal{E}$ denote the clean event where all empirical mean estimates satisfy 
$|\bar{f}(A_i) - f(A_i)| < \text{rad}$ and $|\bar{g}(A_i) - g(A_i)| < \text{rad}$, with 
$\text{rad} = \sqrt{\frac{h^2 \log T}{2m}}$. 

From Lemma \ref{lem:concentration}, 
$\mathbb{P}(\mathcal{E}) \geq 1 - 4NT^{-1}$.

Using the law of total expectation, decompose the $\alpha$-regret and CCV:
\begin{equation}
\mathbb{E}[\mathcal{R}_f(T)] = \mathbb{E}[\mathcal{R}_f(T) | \mathcal{E}] \mathbb{P}(\mathcal{E}) + \mathbb{E}[\mathcal{R}_f(T) | \mathcal{E}^c] \mathbb{P}(\mathcal{E}^c),
\label{eq:regret-decomp}
\end{equation}

\begin{equation}
\mathbb{E}[\mathcal{V}_g(T)] = \mathbb{E}[\mathcal{V}_g(T) | \mathcal{E}] \mathbb{P}(\mathcal{E}) + \mathbb{E}[\mathcal{V}_g(T) | \mathcal{E}^c] \mathbb{P}(\mathcal{E}^c).
\label{eq:ccv-decomp}
\end{equation}

Under $\mathcal{E}$, further decompose both quantities  into separate terms for exploration and exploitation phases:
\begin{eqnarray}
\mathbb{E}[\mathcal{R}_f(T) | \mathcal{E}] &=&
\sum_{i=1}^N m \left(\alpha f(\text{OPT}) - \mathbb{E}[f(S_i)] \right) \nonumber\\&&+ (T - Nm) \left(\alpha f(\text{OPT}) - \mathbb{E}[f(S)] \right),
\label{eq:regret-split}
\end{eqnarray}
\begin{equation}
\mathbb{E}[\mathcal{V}_g(T) | \mathcal{E}] =\sum_{i=1}^N m \left(\mathbb{E}[g(S_i)] - \beta \kappa \right) + (T - Nm) \left(\mathbb{E}[g(S)] - \beta \kappa \right).
\label{eq:ccv-split}
\end{equation}

During the exploration phase, we have the following. 
Since $f(S_i) \leq h$,
\begin{equation}
\sum_{i=1}^N m \left(\alpha f(\text{OPT}) - \mathbb{E}[f(S_i)] \right) \leq \alpha N m h.
\label{eq:regret-explore}
\end{equation}

Since $g(S_i) \leq h$,
\begin{equation}
\sum_{i=1}^N m \left(\mathbb{E}[g(S_i)] - \beta \kappa \right) \leq N m h.
\label{eq:ccv-explore}
\end{equation}

During the exploitation phase, we have the following. Under $\mathcal{E}$, the $\delta$-resilience property ensures:
\begin{equation}
\mathbb{E}[f(S)] \geq \alpha f(\text{OPT}) - \delta \cdot \text{rad}, \quad \mathbb{E}[g(S)] \leq \beta \kappa + \delta \cdot \text{rad}.
\label{eq:resilience}
\end{equation}

Substituting into \eqref{eq:regret-split} and \eqref{eq:ccv-split}:
\begin{equation}
(T - Nm) \left(\alpha f(\text{OPT}) - \mathbb{E}[f(S)] \right) \leq T \delta \cdot \text{rad},
\label{eq:regret-exploit}
\end{equation}

\begin{equation}
(T - Nm) \left(\mathbb{E}[g(S)] - \beta \kappa \right) \leq T \delta \cdot \text{rad}.
\label{eq:ccv-exploit}
\end{equation}

In order to equate exploration and exploitation terms for $m$ in order, we choose
\begin{equation}
m = \Theta\left( \frac{\delta^{2/3} T^{2/3} (\log T)^{1/3}}{N^{2/3}} \right).
\label{eq:m-choice}
\end{equation}

Substituting \eqref{eq:m-choice} into \eqref{eq:regret-explore}, \eqref{eq:ccv-explore}, \eqref{eq:regret-exploit}, and \eqref{eq:ccv-exploit}:
\begin{equation}
\mathbb{E}[\mathcal{R}_f(T) | \mathcal{E}] = \mathcal{O}\left( \delta^{2/3} h N^{1/3} T^{2/3} \log^{1/3} T \right),
\label{eq:regret-bound}
\end{equation}
\begin{equation}
\mathbb{E}[\mathcal{V}_g(T) | \mathcal{E}] = \mathcal{O}\left( \delta^{2/3} h N^{1/3} T^{2/3} \log^{1/3} T \right).
\label{eq:ccv-bound}
\end{equation}

The bad event contribution can be bounded as follows.
For $\mathcal{E}^c$ with $\mathbb{P}(\mathcal{E}^c) \leq 4NT^{-1}$:
\begin{equation}
\mathbb{E}[\mathcal{R}_f(T) | \mathcal{E}^c] \mathbb{P}(\mathcal{E}^c) \leq 4N h = \mathcal{O}(1),
\label{eq:regret-bad}
\end{equation}
\begin{equation}
\mathbb{E}[\mathcal{V}_g(T) | \mathcal{E}^c] \mathbb{P}(\mathcal{E}^c) \leq 4N h = \mathcal{O}(1).
\label{eq:ccv-bad}
\end{equation}

From \eqref{eq:regret-decomp}--\eqref{eq:ccv-decomp}, \eqref{eq:regret-bound}--\eqref{eq:ccv-bound}, and \eqref{eq:regret-bad}--\eqref{eq:ccv-bad}:
\[
\mathbb{E}[\mathcal{R}_f(T)] = \mathcal{O}\left( \delta^{2/3} h N^{1/3} T^{2/3} \log^{1/3} T \right),
\]
\[
\mathbb{E}[\mathcal{V}_g(T)] = \mathcal{O}\left( \delta^{2/3} h N^{1/3} T^{2/3} \log^{1/3} T \right).
\]
\end{proof}

\end{proof}

\section{The Framework Applied to Minimization Problem}\label{app:min}

This appendix details the conversion of the bi-criteria CMAB framework from a maximization problem (Section \ref{sec:problem}) to a minimization problem. We redefine the problem setup and regret guarantees for the minimization setting.

\subsection{Problem Statement for Minimization}
\label{sec:minproblem}
The learner’s goal is to minimize a cumulative \textit{cost} \( \sum_{t=1}^T f_t(A_t) \) while ensuring that the expected \textit{utility} of each action approximately satisfies a lower-bound constraint \( \kappa \in \mathbb{R}^+ \). Formally, we require:  
\[
\frac{1}{T}\sum_{t=1}^T g_t(A_t) \geq \kappa.
\]
Let \( \mathrm{OPT} \) denote the optimal action for the minimization problem:  
\[
\mathrm{OPT} \in \arg\min_{A \subseteq \Omega} f(A) \quad \text{subject to} \quad g(A) \geq \kappa.
\]
The regret and cumulative constraint violation (CCV) are redefined as:  
\[
\mathbb{E}[\mathcal{R}_f(T)] = \mathbb{E}\left[\sum_{t=1}^T f_t(A_t)\right] - \alpha T f(\mathrm{OPT}),
\]
\[
\mathbb{E}[\mathcal{V}_g(T)] = \beta T \kappa - \mathbb{E}\left[\sum_{t=1}^T g_t(A_t)\right],
\]
where \( \alpha \geq 1 \) is the cost approximation factor and \( \beta \leq 1 \) is the utility relaxation factor.

\subsection{Modified Framework and Analysis}
The online algorithm (Algorithm~\ref{alg:bi-criteria-algo}) remains unchanged, but the analysis adapts to the minimization objective:


\begin{theorem}[Regret and CCV for Minimization]\label{thm:main-min}  
For any bi-criteria CMAB minimization instance with horizon \( T \geq \max\left\{\texttt{N}, \frac{2\sqrt{2}\texttt{N}}{\delta}\right\} \) and  $h\triangleq \max(f_{\max},g_{\max})$, \textsc{Bi-Criteria CMAB Algorithm} achieves:  
\begin{enumerate}
    \item Expected \(\alpha\)-regret: 
    \[
      \mathbb{E}[\mathcal{R}_f(T)] = \mathcal{O}\left(\delta^{2/3}{h}\texttt{N}^{1/3}T^{2/3}\log^{1/3}T\right),
    \]
    \item Expected cumulative \(\beta\)-constraint violation: 
    \[
      \mathbb{E}[\mathcal{V}_g(T)] = \mathcal{O}\left(\delta^{2/3}{h}\texttt{N}^{1/3}T^{2/3}\log^{1/3}T\right).
    \]
\end{enumerate}
\end{theorem}

\begin{proof}[Proof Sketch]
The proof follows the same structure as Theorem~\ref{thm:main}, with adjustments for minimization:  

1. \textbf{Clean Event}: Concentration bounds hold as in Lemma \ref{lem:concentration}.  

2. \textbf{Resilience Guarantees}: Under clean event \( \mathcal{E} \),  from Appendix \ref{apd_res_def}, we have
\[
\begin{aligned}
  \mathbb{E}[f(S)] &\leq \alpha f(\mathrm{OPT}) + \delta {h}\mathrm{rad}, \\
  \mathbb{E}[g(S)] &\geq \beta \kappa - \delta {h}\mathrm{rad}.
\end{aligned}
\]  
3. \textbf{Regret and CCV Decomposition}:  
\[
\begin{aligned}
  \mathbb{E}[\mathcal{R}_f(T)|\mathcal{E}] &= \sum_{i=1}^{\texttt{N}} m (\mathbb{E}[f(S_i)] - \alpha f(\mathrm{OPT})) + \sum_{t=\texttt{N}m+1}^T (\mathbb{E}[f(S)] - \alpha f(\mathrm{OPT})), \\
  \mathbb{E}[\mathcal{V}_g(T)|\mathcal{E}] &= \sum_{i=1}^{\texttt{N}} m (\beta \kappa - \mathbb{E}[g(S_i)]) + \sum_{t=\texttt{N}m+1}^T (\beta \kappa - \mathbb{E}[g(S)]).
\end{aligned}
\]  
4. \textbf{Bounding Terms}: Exploration regret \( \leq \texttt{N} m f_{\max} \), exploitation regret \( \leq T \delta  {h}\mathrm{rad} \). Similar bounds apply to CCV, where exploration CCV is \( \leq \texttt{N} m \beta \kappa \) and exploitation CCV \( \leq T \delta {h}\mathrm{rad} \).

5. \textbf{Hyperparameter Substitution}: Substituting \( m = \mathcal{O}\left(\frac{\delta^{2/3}T^{2/3}\log^{1/3}T}{\texttt{N}^{2/3}}\right) \) balances exploration and exploitation terms.  

The full proof mirrors Theorem~\ref{thm:main}, with inequalities modified as above to reflect minimization.  
\end{proof}

\subsection{Key Differences from Maximization Framework}

\begin{enumerate}
\item \textbf{Objective and Constraint Swap}: Cost minimization replaces reward maximization; utility lower-bound replaces cost upper-bound.  
\item \textbf{Regret/CCV Definitions}: Regret measures excess cost, while CCV measures utility shortfall.  
\item \textbf{Resilience Inequalities}: Additive errors \( \delta \epsilon \) increase cost bounds and decrease utility bounds. 
\end{enumerate}

This conversion demonstrates the framework’s flexibility in handling dual objectives across maximization and minimization problems under bandit feedback.

\end{document}